\documentclass{article}
\usepackage{ijcai23}

\usepackage{times}
\usepackage{soul}
\usepackage{url}
\usepackage[hidelinks]{hyperref}
\usepackage[utf8]{inputenc}
\usepackage[small]{caption}
\usepackage{graphicx}
\usepackage{amsmath,amssymb,amsfonts,mathrsfs}
\usepackage{amsthm}
\usepackage{booktabs}
\usepackage{algorithm}
\usepackage{algorithmic}
\usepackage{bm}
\usepackage{subfig}
\usepackage{multirow}
\usepackage{float}

\allowdisplaybreaks[4]

\newtheorem{theorem}{Theorem}
\newtheorem{lemma}[theorem]{Lemma}

\newtheorem{corollary}[theorem]{Corollary}
\newtheorem{definition}[theorem]{Definition}

\usepackage{enumitem}
\setlist{nolistsep}

\DeclareMathOperator*{\argmin}{arg\,min}

\makeatletter
\newenvironment{breakablealgorithm}
  {
   \begin{center}
     \refstepcounter{algorithm}
     \hrule height.8pt depth0pt \kern2pt
     \renewcommand{\caption}[2][\relax]{
       {\raggedright\textbf{\ALG@name~\thealgorithm} ##2\par}%
       \ifx\relax##1\relax 
         \addcontentsline{loa}{algorithm}{\protect\numberline{\thealgorithm}##2}%
       \else 
         \addcontentsline{loa}{algorithm}{\protect\numberline{\thealgorithm}##1}%
       \fi
       \kern2pt\hrule\kern2pt
     }
  }{
     \kern2pt\hrule\relax
   \end{center}
  }
\makeatother

\title{A De-singularity Subgradient Approach \\ for the Extended Weber Location Problem}

\author{
Zhao-Rong Lai$^1$
\and
Xiaotian Wu$^2$\and
Liangda Fang$^2$\And
Ziliang Chen$^{*3,4}$
\affiliations
$^1$Department of Mathematics, College of Information Science and Technology, Jinan University\\
$^2$Department of Computer Science, College of Information Science and Technology, Jinan University\\
$^3$Research Institute of Multiple Agents and Embodied Intelligence, Peng Cheng Laboratory\\
$^4$Guangdong Institute of Smart Education, Jinan University
\emails
\{laizhr, wxiaotian, fangld\}@jnu.edu.cn,
c.ziliang@yahoo.com
}

\begin{document}

\def \bbE {\mathbb E}
\def \bbR {\mathbb R}
\def \bbN {\mathbb N}
\def \bbS {\mathbb S}
\def \bbZ {\mathbb Z}

\def \mD {\mathcal{D}}
\def \mF {\mathcal{F}}
\def \mG {\mathcal{G}}
\def \mI {\mathcal{I}}
\def \mL {\mathcal{L}}
\def \mM {\mathcal{M}}
\def \mN {\mathcal{N}}
\def \mO {\mathcal{O}}
\def \mP {\mathcal{P}}
\def \mQ {\mathcal{Q}}
\def \mS {\mathcal{S}}
\def \mT {\mathcal{T}}
\def \mU {\mathcal{U}}

\def \fa {\mathfrak{a}}
\def \fb {\mathfrak{b}}
\def \frS {\mathfrak{S}}

\def \sL {\mathscr{L}}
\def \sS {\mathscr{S}}

\def \ba {\bm{a}}
\def \bb {\bm{b}}
\def \bc {\bm{c}}
\def \bd {\bm{d}}
\def \bh {\bm{h}}
\def \bp {\bm{p}}
\def \bq {\bm{q}}
\def \bx {\bm{x}}
\def \by {\bm{y}}
\def \bz {\bm{z}}
\def \bw {\bm{w}}
\def \bu {\bm{u}}
\def \bv {\bm{v}}
\def \br {\bm{r}}
\def \bs {\bm{s}}
\def \bR {\bm{R}}
\def \bS {\bm{S}}
\def \bI {\bm{I}}
\def \bA {\bm{A}}
\def \bB {\bm{B}}
\def \bC {\bm{C}}
\def \bD {\bm{D}}
\def \bE {\bm{E}}
\def \bF {\bm{F}}
\def \bG {\bm{G}}
\def \bH {\bm{H}}
\def \bP {\bm{P}}
\def \bQ {\bm{Q}}
\def \bR {\bm{R}}
\def \bU {\bm{U}}
\def \bV {\bm{V}}
\def \bW {\bm{W}}
\def \bX {\bm{X}}

\def \tM {\tilde{M}}

\def \bone {\mathbf{1}}

\def \blambda {\bm{\lambda}}

\def \fp {\mathfrak{p}}
\def \frs {\mathfrak{s}}

\def \bbRn {\bbR^{n}}
\def \bbRm {\bbR^{m}}
\def \bbRN {\bbR^{N}}
\def \bbRM {\bbR^{M}}
\def \bbNM {\bbN_M}
\def \bbNk {\bbN_k}
\def \bbNn {\bbN_n}
\def \bbNm {\bbN_m}
\def \bbZn {\bbZ_n}

\def \tS {\tilde{S}}
\def \tbQ {\tilde{\bQ}}
\def \wtDelta {\widetilde{\Delta}}

\def \st {\text{s.\ t.}}
\def \tr {\mathrm{tr}}
\def \prox {\mathrm{prox}}
\def \supp {\mathrm{supp}}
\def \trun {\mathrm{trun}}
\def \dist {\mathrm{dist}}
\def \Fix {\mathrm{Fix}}
\def \gra {\mathrm{gra}\hspace{2pt}}
\def \dom {\mathrm{dom}\hspace{2pt}}
\def \crit {\mathrm{crit}\hspace{2pt}}

\newcommand\leqs{\leqslant}
\newcommand\geqs{\geqslant}
\newcommand{\ud}{\,\mathrm{d}}

\maketitle


\begin{abstract}
The extended Weber location problem is a classical optimization problem that has inspired some new works in several machine learning scenarios recently. However, most existing algorithms may get stuck due to the singularity at the data points when the power of the cost function $1\leqs q<2$, such as the widely-used iterative Weiszfeld approach. In this paper, we establish a de-singularity subgradient approach for this problem. We also provide a complete proof of convergence which has fixed some incomplete statements of the proofs for some previous Weiszfeld algorithms. Moreover, we deduce a new theoretical result of superlinear convergence for the iteration sequence in a special case where the minimum point is a singular point. We conduct extensive experiments in a real-world machine learning scenario to show that the proposed approach solves the singularity problem, produces the same results as in the non-singularity cases, and shows a reasonable rate of linear convergence. The results also indicate that the $q$-th power case ($1<q<2$) is more advantageous than the $1$-st power case and the $2$-nd power case in some situations. Hence the de-singularity subgradient approach is beneficial to advancing both theory and practice for the extended Weber location problem.
\end{abstract}

\section{Introduction}
\label{sec:intro}
\renewcommand{\thefootnote}{\fnsymbol{footnote}}
\footnotetext[1]{Corresponding author.}
\footnotetext[7]{This is a preprint version of an accepted paper for IJCAI 2024. Contents may be changed in its final version. The supplementary material and code for this paper are available at \url{https://github.com/laizhr/qPWAWS}.}
\renewcommand{\thefootnote}{\arabic{footnote}}
The extended Weber location problem is a classical optimization problem \cite{weiszconvOR,lqnormmean3,weiszoldnew} that has been introduced to some new machine learning scenarios recently \cite{lqmean,olpsjmlr,SSPO,SPOLC,egrmvgap}. It finds the point that minimizes the $q$-th power of the Euclidean distances from $m$ fixed data points $\mathbf{x}_1,\cdots, \mathbf{x}_m\in \mathbb{R}^d$ \cite{l1median1,extendfermat,extendfermat2}:
\begin{equation}
\setlength{\abovedisplayskip}{3pt}
\setlength{\belowdisplayskip}{3pt}
\label{eqn:lqmedian}
\mathbf{x}_*=\argmin_\mathbf{y} C_q(\mathbf{y})=\argmin_\mathbf{y} \sum_{i=1}^m \| \mathbf{y}-\mathbf{x}_i  \|^q,
\end{equation}
where $\|\cdot\|^q$ denotes the $q$-th power of the Euclidean distance, and $C_q(\mathbf{y})$ is the $q$-th power cost function at the point $\mathbf{y}\in \mathbb{R}^d$. The $q$-th power median $\mathbf{x}_*$ (especially when $1\leqs q<2$) has recently been found useful in rotation averaging and gives more robust results when outliers exist \cite{lqmean}. We will further show its advantage in online portfolio selection \cite{olpsjmlr,egrmvgap} in this paper.

\subsection{The Singularity Problem}
\label{sec:singprob}
In order to have a quick insight into the aim of this paper, we first compute the gradient of $C_q(\mathbf{y})$ w.r.t. $\mathbf{y}$:
\begin{equation}
\setlength{\abovedisplayskip}{3pt}
\setlength{\belowdisplayskip}{3pt}
\label{eqn:lqgradient}
\nabla C_q(\mathbf{y})= \sum_{i=1}^m q\| \mathbf{y}-\mathbf{x}_i  \|^{q-2}( \mathbf{y}-\mathbf{x}_i).
\end{equation}
It is well-defined when $q\geqs 2$ and one could easily find some gradient descent methods to work out the minimum \cite{lqmeangrad}, which is beyond the main concern of this paper. However, it is singular at the data points $\{\mathbf{x}_i\}_{i=1}^m$ when $1\leqs q<2$, since at least one of the weights $\| \mathbf{y}-\mathbf{x}_i  \|^{q-2}$ becomes undefined (see Figure \ref{fig:singnonsingular}). This paper mainly addresses this case and we assume $1\leqs q<2$ in the rest of this paper if not specified.

\begin{figure}[!htb]
\vspace{-5pt}
\centering
\subfloat[Nonsingular]{
\centering
\includegraphics[width=0.3\textwidth]{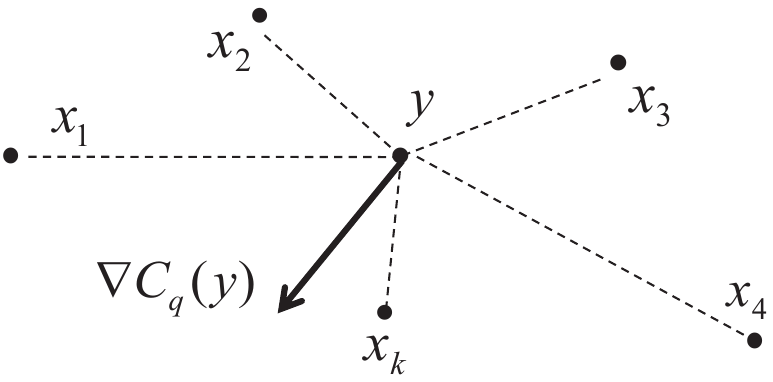}}\\
\subfloat[Singular]{
\centering
\includegraphics[width=0.3\textwidth]{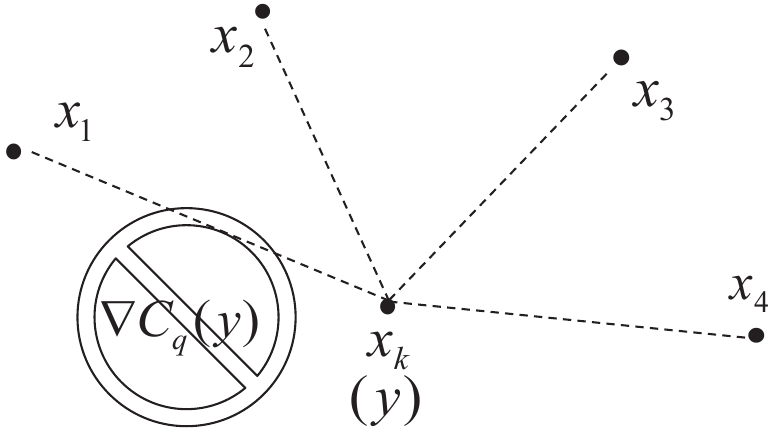}}
\caption{The singularity problem for the extended Weber location problem (\ref{eqn:lqmedian}) with $1\leqs q<2$: (a) The gradient $\nabla C_q(\mathbf{y})$ is well-defined when $\mathbf{y}\notin\{\mathbf{x}_i\}_{i=1}^m$. (b) The gradient $\nabla C_q(\mathbf{y})$ does not necessarily exist when $\mathbf{y}$ hits some $\mathbf{x}_k$.}
\label{fig:singnonsingular}
\vspace{-5pt}
\end{figure}

Since $C_q(\mathbf{y})$ is convex and continuous in $\mathbb{R}^d$ including these data points $\{\mathbf{x}_i\}_{i=1}^m$, we can still turn to the subgradient approach for solutions.
\begin{definition}[\cite{rockafellar2009variational}]
\label{def:frechetsubdifferential}
The Fr{\'e}chet subdifferential of $C_q$ at $\mathbf{y}$, denoted by $\partial C_q(\mathbf{y})$, is the set of all vectors $\mathbf{v}\in\mathbb{R}^d$ satisfying 
\begin{equation}
\label{eqn:frechetsubdifferential}
\setlength{\abovedisplayskip}{3pt}
\setlength{\belowdisplayskip}{3pt}
\partial C_q(\mathbf{y}){\triangleq} \left\{\mathbf{v}\in\mathbb{R}^d: \liminf_{\substack{\bm{z}\to\bm{y}\\ \bm{z}\neq\bm{y}}}\frac{C_q(\mathbf{z}){-}C_q(\mathbf{y}){-}\bm{v}^\top(\bm{z}{-}\bm{y})}{\|\bm{z}-\bm{y}\|_2}{\geqs}0\right\}.
\end{equation}
\end{definition}
Since $C_q$ is also proper ($\dom(C_q)\ne \emptyset$ and $C_q(\mathbf{y})\ne -\infty$, $\forall \mathbf{y}\in\mathbb{R}^d$), its Fr{\'e}chet subdifferential is equivalent to the classical subdifferential in the literature of convex analysis.
\begin{definition}[Subdifferential in Convex Analysis]
\label{def:convsubdifferential}
\begin{equation}
\label{eqn:convsubdifferential}
\partial C_q(\mathbf{y})\triangleq \left\{\mathbf{v}\in\mathbb{R}^d: \forall \mathbf{z}, \  C_q(\mathbf{z})-C_q(\mathbf{y})-\bm{v}^\top(\bm{z}-\bm{y})\geqs0\right\}.
\end{equation}
\end{definition}
If $C_q$ is differentiable at $\mathbf{y}$, $\partial C_q(\mathbf{y})$ reduce to a gradient $\nabla C_q(\mathbf{y})$. Based on the properties of $C_q$ ($1\leqs q<2$), it is easy to verify that at least one Fr{\'e}chet subgradient exists at each singular point $\mathbf{x}_k\in\{\mathbf{x}_i\}_{i=1}^m$ by satisfying the limit inferior inequality in (\ref{eqn:frechetsubdifferential}). However, few existing works have been done to deal with this singularity problem, because it is usually considered a rare event in practice or just overlooked.

One of the most typical and widely-used gradient approaches is the general $q$-th Power Weiszfeld Algorithm ($q$PWA, \cite{extendfermat,extendfermat2,lqmean}). Since $C_q(\mathbf{y})$ is convex, $\mathbf{y}$ is the minimum of $C_q(\mathbf{y})$ if and only if $\nabla C_q(\mathbf{y})=\mathbf{0}$. It leads to the following update formula of $q$PWA at the $p$-th iteration:
\begin{align}
\label{eqn:lqwa2}
\mathbf{y}_{(p+1)}=& \frac{\sum_{i=1}^m \| \mathbf{y}_{(p)}-\mathbf{x}_i  \|^{q-2}\mathbf{x}_i}{\sum_{i=1}^m \| \mathbf{y}_{(p)}-\mathbf{x}_i  \|^{q-2}},
\end{align}
where $\mathbf{y}_{(p)}$ is the current iterate and $\mathbf{y}_{(p+1)}$ is the next iterate computed by a re-weighted sum of the data points $\{\mathbf{x}_i\}_{i=1}^m$. The weights of $\{\mathbf{x}_i\}_{i=1}^m$ are determined by the current distances from $\mathbf{y}_{(p)}$ to $\{\mathbf{x}_i\}_{i=1}^m$ and normalized by their sum.

(\ref{eqn:lqwa2}) can be further transformed to:
\begin{align}
\label{eqn:gradientform}
\mathbf{y}_{(p+1)}=& \mathbf{y}_{(p)}-\frac{\sum_{i=1}^m \| \mathbf{y}_{(p)}-\mathbf{x}_i  \|^{q-2}(\mathbf{y}_{(p)}-\mathbf{x}_i)}{\sum_{i=1}^m \| \mathbf{y}_{(p)}-\mathbf{x}_i  \|^{q-2}}\nonumber\\
\triangleq&\mathbf{y}_{(p)}-\vartheta \nabla C_q(\mathbf{y}_{(p)}),
\end{align}
where $\vartheta=1/(q\sum_{i=1}^m \| \mathbf{y}_{(p)}-\mathbf{x}_i  \|^{q-2})$ is an automatic step size for the gradient descent approach. Therefore, the Weiszfeld algorithm (\ref{eqn:lqwa2}) is actually a gradient descent method (\ref{eqn:gradientform}), which shall fail if $\mathbf{y}_{(p)}$ hits one of $\{\mathbf{x}_i\}_{i=1}^m$ and $1\leqs q<2$.

\subsection{The Significance and Difficulty of This Singularity Problem}
\label{sec:singprob}
This singularity problem actually happens frequently and unexpectedly. We can easily give a simple example in Table \ref {tab:exampleesc}. Let $q=1.1$ and $6$ data points be 
\[
\{\mathbf{x}_i\}_{i=1}^6\triangleq \{(-2,0),(-1,0),(1,0),(2,0),(0,1),(0,-1)\}.
\]
The starting point $\mathbf{y}_{(0)}\triangleq (1.68645,0)$ is chosen distinct from $\{\mathbf{x}_i\}_{i=1}^6$. The existing algorithm $q$PWA gets stuck with just $1$ iteration since it hits one data point $\mathbf{x}_3=(1,0)$ and cannot proceed. The resulted cost function is $C_q=9.4201$, which is not the true minimum. On the contrary, the proposed algorithm $q$PWAWS successfully escapes from the singular point $\mathbf{x}_3$ and finds the true solution $(0,0)$ with the true minimum $C_q=8.2871$. Since $C_q$ is a continuous function on $\{\mathbf{x}_i\}_{i=1}^6$, when some $\mathbf{x}_i$ change continuously, the ``bad'' starting point $\mathbf{y}_{(0)}$ also changes accordingly, possibly throughout the whole $\mathbb{R}^2$. \cite{openweiszfeld,l1median3} further point out that the bad starting point set $\{\mathbf{y}_{(0)}\}$ itself may constitute a continuum set that can be dense in an open region of $\mathbb{R}^d$ with $d\geqslant 2$, even the entire $\mathbb{R}^d$. Hence this singularity problem is quite significant, if not serious.
\begin{table}[!htb]
\vspace{-5pt}
\centering
\small{
\begin{tabular}{cccc}
\hline
$q=1.1$ & Iter & $\hat{\mathbf{x}}$ & $C_q(\hat{\mathbf{x}})$\\
 \hline
$q$PWA & 1 (get stuck) & $(1,0)$ & $9.4201$  \\
$q$PWAWS (ours) & 24 & $(0,0)$ & $8.2871$  \\
  \hline
\end{tabular}
}
\caption{An example of the singularity problem.}
\label{tab:exampleesc}
\vspace{-10pt}
\end{table}

This singularity problem cannot be radically eliminated by straightforward treatments, such as perturbations and random restarts. We remind that $\mathbf{0}$ is not necessarily a subgradient for $C_q(\mathbf{x}_k)$. If $\mathbf{0}\in \partial C_q(\mathbf{x}_k)$, then $C_q(\mathbf{y})-C_q(\mathbf{x}_k) \geqs 0$ for all $\mathbf{y}$ based on (\ref{eqn:convsubdifferential}). Thus $\mathbf{x}_k$ should be a minimum point of $C_q$. However, whether $\mathbf{x}_k$ minimizes $C_q$ is unknown beforehand. Another treatment is to find a starting point with a smaller cost than the costs of any data points $\{\mathbf{x}_i\}_{i=1}^m$ by computing $\{C_q(\mathbf{x}_i)\}_{i=1}^m$ at first \cite{lqmean}, but it is rather exhaustive especially when the number of data points $m$ is large. It requires at least $m$ iterates to find a starting point, but sometimes we only need fewer iterates to find a minimum point. Table \ref{tab:comcost} shows that our method requires only $9.31$ iterates in average to find the minimum in the ``NYSE(N), $m=10$, $q=1.9$'' case, while the above treatment conducts $m=10$ iterates only to find a starting point.Furthermore, if $\mathbf{x}_k$ is exactly the minimum, we have to check whether $\mathbf{0}\in \partial C_q(\mathbf{x}_k)$. After all, once we figure out $\partial C_q(\mathbf{x}_k)$, we can get rid of these troubles. In fact, there is no need to introduce additional treatments because our de-singularity approach can immediately replace the ordinary gradient step without increasing computational complexity, which is a great advantage.

In the $q=1$ case, there is a remedy that removes the singular term in the gradient \cite{kuhnl1median,l1median3}, but it does not investigate the subgradient $\partial C_q(\mathbf{x}_k)$. More importantly, there are some incomplete statements in the proof of convergence, which will be fixed in Section \ref{sec:anaconvergence}. In the $1<q<2$ case, neither has the subgradient $\partial C_q(\mathbf{x}_k)$ been figured out nor has the proof of convergence been completed, yet.

\subsection{Our Results}
\label{sec:ourres}
We mainly establish a complete ``de-singularity'' subgradient approach for the extended Weber location problem (\ref{eqn:lqmedian}) ($1\leqs q<2$). The key contributions fall into four aspects:
\begin{enumerate}[leftmargin=*]
 \item By removing the singular term, we propose a de-singularity subgradient of $C_q$ at each singular point $\mathbf{x}_k\in\{C_q(\mathbf{x}_i)\}_{i=1}^m$. It can replace the ordinary gradient without increasing computational complexity.
 \item If $\mathbf{x}_k$ is not a minimum point, we pull the stuck iterate in the de-singularity subgradient descent direction that can reduce the cost $C_q$, so that the iterates can monotonically converge to the exact minimum. In the rest of this paper, we call the new algorithm $q$-th Power Weiszfeld Algorithm without Singularity ($q$PWAWS).
 \item We present a complete proof of convergence for $q$PWAWS, which has fixed some incomplete statements of the proofs for some previous Weiszfeld algorithms.
 \item If one of $\{\mathbf{x}_i\}_{i=1}^m$ is the minimum, we prove that $q$PWAWS enjoys a superlinear convergence for the iteration sequence when $1<q<2$, which is a new theoretical result. If none of $\{\mathbf{x}_i\}_{i=1}^m$ is the minimum, we show that $q$PWAWS enjoys a reasonable rate of linear convergence for the iteration sequence with computational experiments.
\end{enumerate}

The proposed algorithm $q$PWAWS can use any starting point in $\mathbb{R}^d$ and guarantee a monotonic convergence to the exact minimum. Moreover, this de-singularity subgradient approach can also be adopted in other gradient-related algorithms for problem (\ref{eqn:lqmedian}).

\section{Related Works on the Weiszfeld Algorithms}
\label{sec:problemandrelatework}
In this section, we assume that the data points $\{\mathbf{x}_i\}_{i=1}^m$ are distinct and non-collinear (i.e., they do not lie in a hyperplane). Thus the cost function $C_q(\mathbf{y})$ in (\ref{eqn:lqmedian}) is strictly convex and has a unique minimum. We review some related works on how to derive the Weiszfeld algorithms.

\subsection{Convergence in the Non-singularity Case}
\label{sec:convernonsing}
For the basic $q$PWA (\ref{eqn:lqwa2}), it has been proven that $C_q(\mathbf{y}_{(p+1)})<C_q(\mathbf{y}_{(p)})$ if $\mathbf{y}_{(p+1)}\ne \mathbf{y}_{(p)}$, and that as long as the sequence of iterates does not hit any data points, it monotonically converges to the exact minimum \cite{kuhnl1median,extendfermat,extendfermat2,lqmean}, while some incomplete statements are to be fixed in Section \ref{sec:anaconvergence}. The rate of convergence for $q$PWA ($1<q<2$) has not been deduced, yet.

\subsection{Iterative Re-weighted Least Squares (IRLS) Interpretation}
\label{sec:lqIRLSinterpret}
Based on the current iterate $\mathbf{y}_{(p)}$, a weighted $2$-nd power cost function $\tilde{C}_q(\mathbf{y})$ can be created to approximate $C_q(\mathbf{y})$:
\begin{equation}
\setlength{\abovedisplayskip}{3pt}
\setlength{\belowdisplayskip}{3pt}
\label{eqn:approxcq}
\tilde{C}_q(\mathbf{y})= \sum_{i=1}^m \| \mathbf{y}_{(p)}-\mathbf{x}_i  \|^{q-2}\| \mathbf{y}-\mathbf{x}_i\|^2.
\end{equation}
By setting the gradient of $\tilde{C}_q(\mathbf{y})$ as zero, we have the same update formula as (\ref{eqn:lqwa2}). It is a kind of Iterative Re-weighted Least Squares (IRLS) techniques \cite{irlswty,irlsidc2,irlseld,lqmean}.

\subsection{$1$-st Power Weiszfeld Algorithm without Singularity}
\label{sec:l1medianws}
A remedy for the $L_1$-median is studied in \cite{kuhnl1median,l1median3,weiszconvOR}. First, we define a de-singularity Weiszfeld transform that excludes the singular point as follows:
\begin{equation}
\setlength{\abovedisplayskip}{3pt}
\setlength{\belowdisplayskip}{3pt}
\label{eqn:l1waws}
\tilde{\mathbf{T}}(\mathbf{y})= \frac{\sum_{\mathbf{x}_i\ne \mathbf{y}} \| \mathbf{y}-\mathbf{x}_i  \|^{-1}\mathbf{x}_i}{\sum_{\mathbf{x}_i\ne \mathbf{y}} \| \mathbf{y}-\mathbf{x}_i  \|^{-1}}.
\end{equation}
Then the next iterate is the combination of $\tilde{\mathbf{T}}(\mathbf{y}_{(p)})$ and the current iterate $\mathbf{y}_{(p)}$:
\begin{equation}
\setlength{\abovedisplayskip}{3pt}
\setlength{\belowdisplayskip}{3pt}
\label{eqn:l1waws2}
\mathbf{y}_{(p+1)}= (1-\lambda)\tilde{\mathbf{T}}(\mathbf{y}_{(p)})+\lambda \mathbf{y}_{(p)},
\end{equation}
where $0\leqs\lambda\leqs 1$ is a mixing parameter.

In order to drag the iterate out of the singular point and reduce the cost simultaneously, we adopt the $1$-st power de-singularity subgradient (a formal definition will be given in Definition \ref{defn:desinggrad}) and set $\lambda$ as follows:
\begin{align}
\label{eqn:l1resi}
&\nabla D_1(\mathbf{y}_{(p)})=\sum_{\mathbf{x}_i\ne \mathbf{y}_{(p)}}\|  \mathbf{y}_{(p)}- \mathbf{x}_i\|^{-1} ( \mathbf{y}_{(p)}-\mathbf{x}_i),\\
\label{eqn:l1lbdopt}
&\lambda=\left\{  \begin{array}{ll}
0 & \text{if} \quad\mathbf{y}_{(p)}\notin \{\mathbf{x}_i\}_{i=1}^m\\
\min \left\{ 1, \frac{1}{\|\nabla D_1(\mathbf{y}_{(p)})\|} \right\} & \text{if} \quad\mathbf{y}_{(p)}\in \{\mathbf{x}_i\}_{i=1}^m
\end{array} \right..
\end{align}
Compared with (\ref{eqn:lqgradient}), $\nabla D_1(\mathbf{y}_{(p)})$ removes the singular term in $\nabla C_1(\mathbf{y}_{(p)})$. If $\mathbf{y}_{(p)}$ does not hit any data point, then (\ref{eqn:l1waws2}) is just the same as (\ref{eqn:lqwa2}). Else, $\lambda$ ensures $C_1(\mathbf{y}_{(p+1)})\leqs C_1(\mathbf{y}_{(p)})$ and the iterates monotonically converge to the exact minimum, which are explained by \cite{l1median3}.

\section{$q$-th Power Weiszfeld Algorithm without Singularity}
\label{sec:lqwaws}
Although there is a remedy for the $1$-st power case, it is nontrivial to deal with the general $q$-th power ($1<q<2$) case. We will see that the $1$-st power case and the $q$-th power ($1<q<2$) case are different in the characterization of minimum, the minimizing strategy, and the rate of convergence, since they have different smoothness. Nevertheless, we still present a unified $q$-th power ($1\leqs q<2$) framework for the integrity of the approach.

In the rest of this paper, we generalize the $q$-th power ($1\leqs q<2$) cost function in (\ref{eqn:lqmedian}) to:
\begin{equation}
\setlength{\abovedisplayskip}{3pt}
\setlength{\belowdisplayskip}{3pt}
\label{eqn:lqmediangen}
 C_q(\mathbf{y})=\sum_{i=1}^m \xi_i\| \mathbf{y}-\mathbf{x}_i  \|^q,
\end{equation}
where $\xi_i>0$ is the multiplicity of the data point $\mathbf{x}_i$. It can deal with duplication and collinearity of data points. For example, if $\mathbf{x}_1=\mathbf{x}_2=\mathbf{x}_3$ and the other points are distinct, then we can merge $\mathbf{x}_1, \mathbf{x}_2, \mathbf{x}_3$ and set $\xi_1=3$, $\xi_i=1 (i\geqs 4)$. The collinear data points can also be merged to be non-collinear ones by the same way. Hence, we assume that the data points are distinct and non-collinear in the rest of this paper. Thus $C_q(\mathbf{y})$ is strictly convex on $\mathbf{y}$ and there is a unique minimum point $\mathbf{M}$. To be convenient for illustrations, we set $\eta_i\triangleq\xi_i^{\frac{1}{q}}$ and further change (\ref{eqn:lqmediangen}) to:
\begin{equation}
\setlength{\abovedisplayskip}{3pt}
\setlength{\belowdisplayskip}{3pt}
\label{eqn:lqmediangenreal}
 C_q(\mathbf{y})=\sum_{i=1}^m \eta_i^q\| \mathbf{y}-\mathbf{x}_i  \|^q.
\end{equation}

The illustration of $q$PWAWS consists of $5$ steps: 
\begin{enumerate}[leftmargin=*]
\item The general update formula is introduced as an easy beginning. 
\item The $q$-th power de-singularity subgradient is defined to characterize the subgradients and the minimum.
\item The $q$-th power de-singularity subgradient is adopted to get out of the singular point and reduce the cost simultaneously.
\item The proof of convergence is conducted.
\item The theoretical rate of convergence for the iteration sequence in a special case is deduced with the properties of the $q$-th power de-singularity subgradient.
\end{enumerate}

\subsection{General Update Formula without Coincidence}
\label{sec:genupdatewc}
First, we consider the simplest case where the current iterate $\mathbf{y}_{(p)}$ does not coincide with the data points $\{\mathbf{x}_i\}_{i=1}^m$. We first state the general iterative update formula in this case:
\begin{align}
\label{eqn:lqwa3}
\mathbf{y}_{(p+1)}= \mathbf{T}_1( \mathbf{y}_{(p)})\triangleq\frac{\sum_{i=1}^m \eta_i^q\| \mathbf{y}_{(p)}-\mathbf{x}_i  \|^{q-2}\mathbf{x}_i}{\sum_{i=1}^m \eta_i^q\| \mathbf{y}_{(p)}-\mathbf{x}_i  \|^{q-2}}.
\end{align}
As a more general case than \cite{kuhnl1median,extendfermat,extendfermat2,lqmean}, we present the following non-increasing theorem:
\begin{theorem}
\label{thm:nonincreasing}
If $\mathbf{y}_{(p)}\notin \{\mathbf{x}_i\}_{i=1}^m$, then $C_q(\mathbf{T}_1( \mathbf{y}_{(p)}))\leqs C_q(\mathbf{y}_{(p)})$ with equality only when $\mathbf{T}_1( \mathbf{y}_{(p)})= \mathbf{y}_{(p)}$.
\end{theorem}
The proof is provided in Supplementary \ref{proof:nonincreasing}.

\subsection{Characterization of Subgradients and Minimum}
\label{sec:charmin}
Second, we establish the de-singularity subgradient and characterize the minimum point $\mathbf{M}$ of $C_q(\mathbf{y})$ in (\ref{eqn:lqmediangenreal}). The following corollary is a direct result of Theorem \ref{thm:nonincreasing}.

\begin{corollary}
\label{cor:charnonsing}
If $\mathbf{y}_{(p)}\notin \{\mathbf{x}_i\}_{i=1}^m$, then $\mathbf{T}_1( \mathbf{y}_{(p)})= \mathbf{y}_{(p)} \Leftrightarrow$ $\mathbf{y}_{(p)}$ is the minimum point $\mathbf{M}$ of $C_q(\mathbf{y})$.
\end{corollary}
The proof is provided in Supplementary \ref{proof:charnonsing}. If $\mathbf{y}_{(p)} \in \{\mathbf{x}_i\}_{i=1}^m$, then the minimum characterization relies on the subgradient(s) in $\partial C_q(\mathbf{y}_{(p)})$. Without loss of generality, suppose $\mathbf{y}_{(p)}=\mathbf{x}_k$. 
\begin{definition}[$\mathbf{q}$\textbf{-th Power De-singularity Subgradient}]
\label{defn:desinggrad}
Let $D_q(\mathbf{y})$ be the main component of $C_q(\mathbf{y})$ that excludes the term $\eta_k^q\| \mathbf{y}-\mathbf{x}_k  \|^q$. Then the $q$-th power de-singularity subgradient $\nabla D_q(\mathbf{y})$ of $C_q(\mathbf{y})$ is:
\begin{align}
\label{eqn:dq}
D_q(\mathbf{y})&\triangleq\sum_{i\ne k} \eta_i^q\| \mathbf{y}-\mathbf{x}_i  \|^q,\\
\label{eqn:dqgrad}
\nabla D_q(\mathbf{y})&{=}\sum_{i\ne k} q\eta_i^q\| \mathbf{y}{-}\mathbf{x}_i  \|^{q-2}(\mathbf{y}{-}\mathbf{x}_i),\quad 1\leqs q<2.
\end{align}
\end{definition}
Compared with the ordinary gradient (\ref{eqn:lqgradient}), the de-singularity subgradient (\ref{eqn:dqgrad}) does not require more computation but checking whether $\mathbf{y}=\mathbf{x}_k$ for the $k$-th summand, which can be done at the same time and in the same loop with summand computing. Hence the de-singularity subgradient will not increase computational complexity.

\begin{theorem}[Characterization of Subgradients and Minimum]
\label{thm:charsing}
Let $\mathbf{y}_{(p)}=\mathbf{x}_k$. Then 
\begin{align}
\label{eqn:subgrad}
\partial C_q(\mathbf{x}_k)&{=}\left\{  \begin{array}{@{}ll}
\{\nabla D_1(\mathbf{x}_k){+} \eta_k \mathbf{u}:  \forall\mathbf{u}, \|\mathbf{u}\|{\leqs} 1\}   &    \text{if}  \quad q{=}1   \\
\{\nabla D_q(\mathbf{x}_k)\} &    \text{if} \quad 1{<} q{<}2
\end{array} \right..
\end{align}
According to Fermat's rule, $\mathbf{x}_k$ is the minimum point $\mathbf{M}$ if and only if $\mathbf{0}\in \partial C_q(\mathbf{x}_k)$.
\end{theorem}
The proof is provided in Supplementary \ref{proof:charsing}. It is obvious that $\nabla D_q(\mathbf{x}_k)\in \partial C_q(\mathbf{x}_k)$ for all $1\leqs q<2$. This is why $\nabla D_q(\mathbf{x}_k)$ is termed a de-singularity subgradient: it removes the singularity and then becomes a subgradient. One can also see that (\ref{eqn:l1resi}) is a special case of (\ref{eqn:subgrad}) with $\mathbf{y}_{(p)}=\mathbf{x}_k$ and $\mathbf{u}=\mathbf{0}$. 

$-\nabla D_q(\mathbf{x}_k)$ can be interpreted as the resultant implemented on $\mathbf{x}_k$ towards other data points (see Figure \ref{fig:resultant}). When $q=1$ and $\|\nabla D_1(\mathbf{x}_k)\|$ is smaller than the intrinsic force $\eta_k$ implemented on $\mathbf{x}_k$, the system remains still and $C_q(\mathbf{x}_k)$ has reached its minimum. But when $1<q<2$, $\|\nabla D_q(\mathbf{x}_k)\|$ should be zero to keep the system still, no matter how large the intrinsic force $\eta_k$ is.

\begin{figure}[!htb]
\vspace{-5pt}
\centering
\includegraphics[width=0.8\columnwidth]{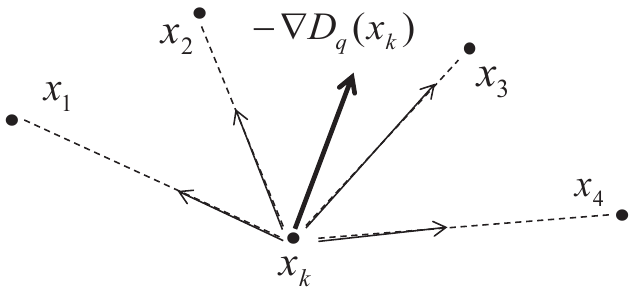}
\caption{$-\nabla D_q(\mathbf{x}_k)$ can be interpreted as the resultant implemented on $\mathbf{x}_k$ towards other data points.}
\label{fig:resultant}
\vspace{-5pt}
\end{figure}

We have characterized the minimum no matter whether the current iterate $\mathbf{y}_{(p)}$ hits the data points $\{\mathbf{x}_i\}_{i=1}^m$ or not. If $\mathbf{y}_{(p)}$ has not reached the minimum point $\mathbf{M}$, then another iteration should be implemented. When $\mathbf{y}_{(p)}\notin \{\mathbf{x}_i\}_{i=1}^m$, the general formula (\ref{eqn:lqwa3}) can be adopted. When $\mathbf{y}_{(p)}\in \{\mathbf{x}_i\}_{i=1}^m$, we will show later that $\mathbf{y}_{(p)}-\lambda \nabla D_q(\mathbf{y}_{(p)})$ can get away from the singular point and reduce the $q$-th power cost simultaneously.

\subsection{Update Formula with Coincidence}
\label{sec:updatecoin}
The update formula with coincidence $\mathbf{y}_{(p)}=\mathbf{x}_k$ should be set up separately for $q=1$ and $1<q<2$. The $q=1$ case is provided in Supplementary \ref{sec:updatecoinqone}. Now we turn to solve the $1<q<2$ case. The following theorem shows a way to reduce the $q$-th power cost when the iterate gets stuck in the singular point.
\begin{theorem}[De-singularity Subgradient Descent Method]
\label{thm:singmin}
If $\mathbf{y}_{(p)}=\mathbf{x}_k$ and $\|\nabla D_q(\mathbf{x}_k)\|> 0$, then there exists a $\lambda_*>0$ such that for any $0<\lambda\leqs\lambda_*$, $C_q(\mathbf{x}_k-\lambda \nabla D_q(\mathbf{x}_k))<C_q(\mathbf{x}_k)$, $1<q<2$.
\end{theorem}
The proof is provided in Supplementary \ref{proof:singmin}. The key point of Theorem \ref{thm:singmin} is that the negative de-singularity subgradient part (resultant) $-\lambda \| \nabla D_q(\mathbf{x}_k) \|^2$ is of lower order infinitesimal than the positive singular part (resistance) $\eta_k^q\lambda^q\| \nabla D_q(\mathbf{x}_k)  \|^q$. Thus when $\lambda$ is sufficiently small, the resultant overcomes the resistance and drags the current iterate $\mathbf{y}_{(p)}$ out of the singular point and the $q$-th power cost is reduced simultaneously. An algorithm can be designed based on this mechanism. By omitting $\frac{\lambda}{2}G(\mathbf{x}_k)+o(\lambda)$ on the left side of (\ref{eqn:cqdecom4}), we can approximately start with the following $\lambda_0$:
\begin{align}
\label{eqn:initlbd0}
\lambda_0= \min\left\{\frac{1}{q}\eta_k^{-\frac{q}{q-1}}\| \nabla D_q(\mathbf{x}_k) \|^{\frac{2-q}{q-1}},1\right\}.
\end{align}
As long as $C_q(\mathbf{x}_k-\lambda_w \nabla D_q(\mathbf{x}_k))\geqs C_q(\mathbf{x}_k)$, we reduce $\lambda_w$ with a factor $\rho<1$: $\lambda_{w+1}\leftarrow \rho\lambda_w$, $w=0,1,\cdots$, until we find some $\lambda_*$ such that $C_q(\mathbf{x}_k-\lambda_* \nabla D_q(\mathbf{x}_k))< C_q(\mathbf{x}_k)$ (see Figure \ref{fig:escapesing}). According to Theorem \ref{thm:singmin}, this $\lambda_*$ is sure to be found. Then the next iterate is:
\begin{align}
\label{eqn:singnextiter}
\mathbf{y}_{(p+1)}=\mathbf{T}_2( \mathbf{x}_k)\triangleq\mathbf{x}_k-\lambda_* \nabla D_q(\mathbf{x}_k).
\end{align}
This strategy is very efficient to get away from the singular point: it takes only $2.75$ iterates in average in the worst case of our experiments (see Table \ref{tab:escapesing}). It can be even more efficient if we further reduce $\lambda_0$ and $\rho$. Besides, the choice of $\lambda_*$ can also be absorbed in momentum acceleration methods, such as Nesterov \cite{nester0,nester1} and Adam \cite{Adam}. Details can be found in the code link provided in the section of acknowledgments.

\begin{figure}[!htb]
\vspace{-5pt}
\centering
\includegraphics[width=0.8\columnwidth]{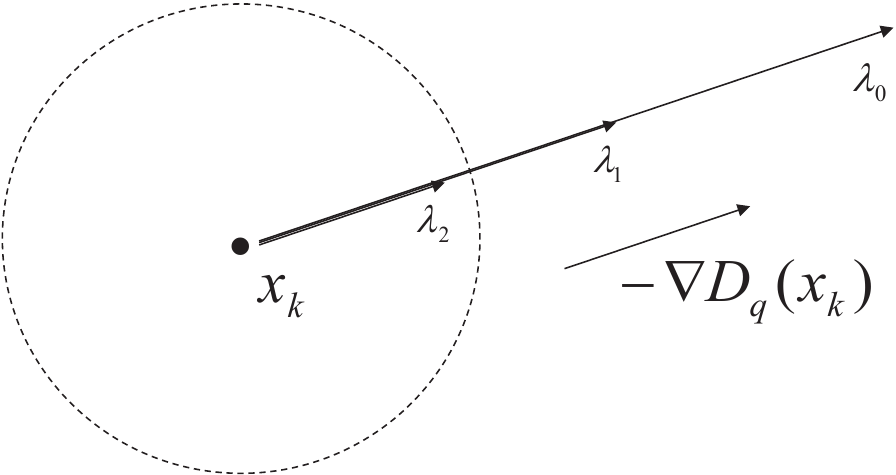}
\caption{Theorem \ref{thm:singmin} indicates that a small displacement from $\mathbf{x}_k$ towards $-\nabla D_q(\mathbf{x}_k)$ can reduce the $q$-th power cost. Thus we can start from a $\lambda_0$ and reduce it with a factor $\rho<1$ at each time, until $\mathbf{x}_k-\lambda_* \nabla D_q(\mathbf{x}_k)$ reduces the cost.}
\label{fig:escapesing}
\vspace{-5pt}
\end{figure}

\subsection{A Complete Proof of Convergence}
\label{sec:anaconvergence}
From (\ref{eqn:lqwa3}) and (\ref{eqn:singnextiter}), the $q$-th Power Weiszfeld Algorithm without Singularity ($q$PWAWS) can be designed as ($1<q<2$):
\begin{align}
\label{eqn:lqwagen}
\mathbf{y}_{(p+1)}{=}\mathbf{T}( \mathbf{y}_{(p)}){\triangleq}\left\{\begin{array}{l@{}l}
\mathbf{T}_1( \mathbf{y}_{(p)})=&\frac{\sum_{i=1}^m \eta_i^q\| \mathbf{y}_{(p)}-\mathbf{x}_i  \|^{q-2}\mathbf{x}_i}{\sum_{i=1}^m \eta_i^q\| \mathbf{y}_{(p)}-\mathbf{x}_i  \|^{q-2}} \\
&\qquad\text{if}\: \mathbf{y}_{(p)}\notin \{\mathbf{x}_i\}_{i=1}^m\\
\mathbf{T}_2( \mathbf{y}_{(p)})=&\mathbf{x}_k{-}\lambda_* \nabla D_q(\mathbf{x}_k) \\
&\qquad\text{if}\: \mathbf{y}_{(p)}=\mathbf{x}_k
\end{array} \right.,
\end{align}
where $\mathbf{T}_1( \mathbf{y}_{(p)})$ is the update formula without coincidence (\ref{eqn:lqwa3}) and $\mathbf{T}_2( \mathbf{y}_{(p)})$ is the update formula with coincidence (\ref{eqn:singnextiter}). Note that the $q=1$ algorithm without singularity has been proposed by \cite{kuhnl1median,l1median3}, but their convergence proofs are incomplete. Thus we focus on the proof of the $1<q<2$ case as well as fix their incomplete statements in this subsection. In calculation, the $q=1$ case and the $1<q<2$ case can be unified.

Starting with any initial point $\mathbf{y}_{(0)}$, the $q$PWAWS generates a sequence of iterates $\mathbf{y}_{(0)},\mathbf{y}_{(1)},\cdots,\mathbf{y}_{(p)}\cdots$. It is necessary to prove that the sequence converges to the minimum point $\mathbf{M}$ unless some $\mathbf{y}_{(p)}$ hits $\mathbf{M}$.

\begin{lemma}
\label{lem:convexhull}
The sequence $\{\mathbf{y}_{(p)}\}$ generated by $q$PWAWS (\ref{eqn:lqwagen}) visits each $\mathbf{x}_k\ne \mathbf{M}$ at most once and will not get stuck. Except for at most a finite set of iterates, $\{\mathbf{y}_{(p)}\}$ and $\mathbf{M}$ lie in the convex hull of the data points $\{\mathbf{x}_i\}_{i=1}^m$.
\end{lemma}
The proof is provided in Supplementary \ref{proof:convexhull}. It is also possible that the sequence $\{\mathbf{y}_{(p)}\}$ converges to the data points $\{\mathbf{x}_i\}_{i=1}^m$. In fact, $\mathbf{T}_1(\mathbf{y})$ is continuous in almost the whole $\mathbb{R}^d$ except for the data points $\{\mathbf{x}_i\}_{i=1}^m$ where it is discontinuous. However, $\{\mathbf{x}_i\}_{i=1}^m$ are removable discontinuities. Less general versions of this characteristic are mentioned in \cite{kuhnl1median,lqmean} but without proof. Thus we present the following lemma:
\begin{lemma}
\label{lem:contintheorem}
\begin{align}
\label{eqn:continlem}
 \lim_{\mathbf{y}\rightarrow \mathbf{x}_k}\mathbf{T}_1(\mathbf{y})=\mathbf{x}_k, \quad\forall 1\leqs k\leqs m, \quad\forall 1\leqs q<2.
\end{align}
\end{lemma}
The proof is provided in Supplementary \ref{proof:contintheorem}. Based on this lemma, the operator $\mathbf{T}_1$ can be extended to the removable discontinuous data points:
\begin{align}
\label{eqn:t1extend}
\mathbf{T}_1(\mathbf{x}_k)\triangleq\mathbf{x}_k, \quad \forall 1\leqs k\leqs m, \quad\forall 1\leqs q<2.
\end{align}
From (\ref{eqn:t1extend}) and Corollary \ref{cor:charnonsing}, all the fixed points of $\mathbf{T}_1$ are characterized as:
\begin{align}
\label{eqn:t1fixedpoint}
\mathbf{T}_1(\mathbf{y})=\mathbf{y}\quad\Longleftrightarrow \quad\mathbf{y}\in \{\mathbf{x}_i\}_{i=1}^m\bigcup \{\mathbf{M}\}.
\end{align}

When $\mathbf{y}$ gets into some neighborhood of a data point $\mathbf{x}_k\ne \mathbf{M}$, the operator $\mathbf{T}_1$ will
eventually drive it out of the neighborhood. This property of $\mathbf{T}_1$ helps to find the real minimum point. The following lemma is a nontrivial extension of a similar conclusion for the $1$-st power case in \cite{kuhnl1median}.

\begin{lemma}
\label{lem:kickout}
If $\mathbf{x}_k\ne \mathbf{M}$, then there exists some $\delta_0>0$ such that for all $\mathbf{y}\in B(\mathbf{x}_k,\delta_0)$, $\mathbf{T}_1^{s-1}(\mathbf{y})\in B(\mathbf{x}_k,\delta_0)$ and $\mathbf{T}_1^{s}(\mathbf{y})\notin B(\mathbf{x}_k,\delta_0)$ for some $s$. $B(\mathbf{x}_k,\delta_0)$ denotes an open $\delta_0$-ball centered at $\mathbf{x}_k$. $\mathbf{T}_1^{s}$ means implementing $\mathbf{T}_1$ for $s$ times, and $s$ depends on $\mathbf{y}$.
\end{lemma}
The proof is provided in Supplementary \ref{proof:kickout}.

\begin{theorem}[Convergence Theorem]
\label{thm:contheorem}
Starting from \textbf{any initial point} $\mathbf{y}_{(0)}$, if the sequence $\{\mathbf{y}_{(p)}\}$ generated by $q$PWAWS (\ref{eqn:lqwagen}) does not hit $\mathbf{M}$, then $\mathbf{y}_{(p)}\rightarrow \mathbf{M}$. If $\mathbf{y}_{(p)}$ hits $\mathbf{M}$, the characterization of minimum (Theorem \ref{thm:charsing}) ensures that this hit could be recognized and the algorithm would be stopped.
\end{theorem}
The proof is provided in Supplementary \ref{proof:contheorem}.

\textbf{Some notes}: 1. In \cite{kuhnl1median,l1median3}, the proof of convergence assumes that \\ $\lim_{v\rightarrow \infty}\mathbf{y}_{(p_v)}=\mathbf{x}_k$ but $\mathbf{T}_1(\mathbf{y}_{(p_v)})\notin B(\mathbf{x}_k,\delta_0)$ for all $v$. However, Lemma \ref{lem:kickout} indicates that $\mathbf{T}_1$ may be implemented for $s>1$ times to drive $\mathbf{y}_{(p_v)}$ out of $B(\mathbf{x}_k,\delta_0)$. Besides, this $s$ also depends on $\mathbf{y}_{(p_v)}$ and the uniform upper bound of $s$ for all the $\mathbf{y}_{(p_v)}$ may not exist. Thus it may not be concluded that $\lim_{v\rightarrow \infty}\frac{\|\mathbf{T}_1(\mathbf{y}_{(p_v)})-\mathbf{x}_k\|}{\|\mathbf{y}_{(p_v)}-\mathbf{x}_k\|}=\infty$ and it may fail the conclusions in \cite{kuhnl1median,l1median3}. Moreover, \cite{kuhnl1median,l1median3} have only shown that a subsequence $\{\mathbf{y}_{(p_v)}\}$ converges to $\mathbf{M}$, but not the whole sequence $\{\mathbf{y}_{(p)}\}$. The convergence of the whole sequence $\{\mathbf{y}_{(p)}\}$ should depend on the uniqueness of $\mathbf{M}$.

2. In \cite{lqmean}, the proof of convergence does not include Lemma \ref{lem:kickout}. Thus it may not guarantee that a subsequence $\mathbf{y}_{(p_u)}\in B(\mathbf{x}_k,\delta_0)$ and $\mathbf{T}_1(\mathbf{y}_{(p_u)})\notin B(\mathbf{x}_k,\delta_0)$. In this case, (\ref{eqn:ystar12}), (\ref{eqn:ystar12b}), (\ref{eqn:ystar12c}) and (\ref{eqn:ystar12d}) may not be deduced and the contradiction may not exist.

Therefore, the proof of convergence in this paper has also fixed some incomplete statements of some related works before.

\subsection{Rate of Convergence}
\label{sec:converatespecial}
It is also important to analyze the rate of convergence for $q$PWAWS when $\mathbf{y}_{(p)}\rightarrow \mathbf{M}$. With the proposed de-singularity subgradient $\nabla D_q$ in (\ref{eqn:dqgrad}), we can deduce the exact rate of convergence for $q$PWAWS in the special case $\mathbf{M}\in\{\mathbf{x}_i\}_{i=1}^m$. To the best of our knowledge, this is a new theoretical result that $q$PWAWS enjoys a superlinear convergence for the iteration sequence in this case. Without loss of generality, we assume $\mathbf{M}=\mathbf{x}_k$. Since the subgradients are different between the $q=1$ case and the $1<q<2$ case (Theorem \ref{thm:charsing}), we further divide the analysis into two parts. The $q=1$ case has been deduced by \cite{weiszconvOR}, while the $1<q<2$ case is somewhat complicated. The key technique for the $1<q<2$ case is to check the order of infinitesimal of $\|\nabla D_q(\mathbf{y}_{(p)})\|$ with $\mathbf{y}_{(p)}\rightarrow \mathbf{x}_k$ at first.
\begin{lemma}
\label{lem:orderinfinitesimal}
If $1<q<2$, the order of infinitesimal of $\|\nabla D_q(\mathbf{y}_{(p)})\|$ is no higher than that of $\|\mathbf{y}_{(p)}-\mathbf{x}_k\|$ with $\mathbf{y}_{(p)}\rightarrow \mathbf{x}_k$. Precisely, there exists some $\zeta\geqs 0$ such that
\begin{equation}
\setlength{\abovedisplayskip}{3pt}
\setlength{\belowdisplayskip}{3pt}
\label{eqn:orderinfinitesimal}
\lim_{\mathbf{y}_{(p)}\rightarrow \mathbf{x}_k}\frac{\|\nabla D_q(\mathbf{y}_{(p)})\|}{\|\mathbf{y}_{(p)}-\mathbf{x}_k\|} \leqs \zeta,\quad 1<q<2.
\end{equation}
\end{lemma}
The proof is provided in Supplementary \ref{proof:orderinfinitesimal}.

\begin{theorem}
\label{thm:rateconver}
If $\mathbf{M}=\mathbf{x}_k$ for some $k$, the rate of convergence for $q$PWAWS is:
\begin{align}
\label{eqn:rateconver}
\lim_{\mathbf{y}_{(p)}\rightarrow \mathbf{x}_k}\frac{\|\mathbf{y}_{(p+1)}-\mathbf{x}_k\|}{\|\mathbf{y}_{(p)}-\mathbf{x}_k\|} &=\left\{\begin{array}{ll}
\frac{\|\nabla D_1(\mathbf{x}_k)\|}{\eta_k} & q=1,\\
 0  & 1<q<2.
\end{array} \right.
\end{align}
It is a superlinear convergence when $1<q<2$ and no worse than a sublinear convergence when $q=1$.
\end{theorem}
The proof is provided in Supplementary \ref{proof:rateconver}. As for the general case $\mathbf{M}\notin\{\mathbf{x}_i\}_{i=1}^m$, the rate of convergence is unknown at present according to our knowledge. It is difficult to eliminate the infinitesimal $\|\mathbf{y}_{(p)}-\mathbf{M}\|$ and obtain a uniform upper bound of $\lim_{\mathbf{y}_{(p)}\rightarrow \mathbf{M}}\frac{\|\mathbf{y}_{(p+1)}-\mathbf{M}\|}{\|\mathbf{y}_{(p)}-\mathbf{M}\|}$ that is no greater than $1$. When $q=1$, \cite{weiszconvOR} shows some computational evidence that the rate of convergence for $\mathbf{M}\notin\{\mathbf{x}_i\}_{i=1}^m$ is usually somewhat less than that of $\mathbf{M}\in\{\mathbf{x}_i\}_{i=1}^m$. We will give some computational results for the $1<q<2,\mathbf{M}\notin\{\mathbf{x}_i\}_{i=1}^m$ case in Section \ref{sec:converateexperiment} to show that $q$PWAWS enjoys a reasonable rate of linear convergence. We summarize the whole $q$PWAWS in Supplementary \ref{sec:solvealgo}, which gives a complete procedure to deal with different situations that may occur in real applications.

\section{Experimental Results}
\label{sec:experiment}
Since $1$PWA and $q$PWA have already been verified to be effective in many applications \cite{rotspace1,rotspace2,genrie,genrie2,RMR2}, we do not intend to repeat their experiments. Instead, we focus on validating that $q$PWAWS successfully solves the singular problem and enjoys a reasonable rate of linear convergence.

We conduct experiments on an interesting machine learning application: online portfolio selection (OPS, \cite{olpsjmlr,SSPO,SPOLC,egrmvgap}). By using the notations in this paper, a data point $\mathbf{x}_i\in \mathbb{R}^d$ indicates a price vector of $d$ assets on the $i$-th day. $\{\mathbf{x}_i\}_{i=1}^m$ contains the asset prices on the most recent $m$ days. Following \cite{RMR2}, we fed the $q$-th power median
\begin{equation}
\setlength{\abovedisplayskip}{3pt}
\setlength{\belowdisplayskip}{3pt}
\label{eqn:lqmedianexp}
\hat{\mathbf{x}}=\argmin_\mathbf{y} \sum_{i=1}^m \| \mathbf{y}-\mathbf{x}_i  \|^q, \quad 1\leqs q\leqs 2
\end{equation}
to a portfolio optimization model \cite{RMR2} to decide the future portfolio. We adopt the NYSE(N) data set \cite{CWMR} and propose a new CSI300 data set.
\begin{itemize}[leftmargin=*]
\item NYSE(N): it contains daily price relative sequences of $d=23$ stocks from New York Stock Exchange during 1/Jan/1985$\sim$30/Jun/2010 ($T=6431$ days).
\item CSI300: it contains daily price relative sequences of $d=47$ stocks from the CSI300 constituents\footnote{\href{http://www.csindex.com.cn}{http://www.csindex.com.cn}} of Shanghai Stock Exchange and Shenzhen Stock Exchange in China during 16/Mar/2015$\sim$19/May/2017 ($T=534$ days).
\end{itemize}

Experiments consist of four parts: first, we verify that $q$PWAWS solves the singularity problem and obtains the same results as $q$PWA. Second, we measure the computational cost and the number of iterations for convergence of $q$PWAWS. Third, we give computational results on the rate of convergence for $q$PWAWS. Fourth, we show that the $q$-th power median ($1<q<2$) is more effective and robust than the $1$-st power median and the $2$-nd power median in some cases, thus the $q$-th power median is useful and it is important to solve the singularity problem. As for the parameters, if not specified, the observation window size is $m=5$, which is consistent with previous related methods \cite{RMR2,PPT,AICTR,mtcvar}. The tolerance threshold is $Tol=10^{-9}$. The reducing factor is $\rho=0.1$, which is a moderate value. As the observation window moves from $t=1$ to $t=T-m+1$, there are a total number of $(T-m+1)$ sets of data points $\{\mathbf{x}_i\}_{i=1}^m$. Thus the $q$-th power median computation can be conducted for $(T-m+1)$ times to evaluate the average performance of $q$PWAWS.

\subsection{Addressing the Singularity Problem}
\label{sec:expsolvesing}
If the sequence of iterates does not hit the data points, then $q$PWAWS is equivalent to $q$PWA. Without loss of generality, we examine a specific case: let the starting point $\mathbf{y}_{(0)}=\mathbf{x}_1$ for $q$PWAWS and $\mathbf{y}_{(0)}=\frac{1}{m}\sum_{i=1}^m \mathbf{x}_i$ for $q$PWA. Besides, we also watch to avoid getting stuck for $q$PWA, since it cannot deal with singularity.

To evaluate the efficiency of $q$PWAWS, we measure how many iterates that $q$PWAWS takes to successfully get away from the singular point with a reduced cost (implementing only Step 11 in Algorithm \ref{alg:lqwaws}). For each data set, we use two different sizes of observation windows $m=5$ and $m=10$, and let $q=1.1\sim 1.9$. The mean and the standard deviation (STD) of the number of iterates are shown in Table \ref{tab:escapesing}. $q$PWAWS successfully gets away from the singular points in all the experiments. As $q$ increases from $1.1$ to $1.9$, the average number of iterates decreases from about $2.54$ to $2$ for $m=5$ and from about $2.75$ to $2$ for $m=10$ on CSI300. In fact, a smaller $\rho$ and a smaller $\lambda_0$ in (\ref{eqn:initlbd0}) can further reduce the number of iterates and Step 11 only need to compute the $q$-th power cost. In general, the number of iterates to get out of the singular point is small.

\begin{table}[!htb]
\vspace{-5pt}
\small{
\centering
\scalebox{0.9}{
\begin{tabular}{ccccc}
\hline
 \multirow{2}{*}{$q$} & \multicolumn{2}{c}{NYSE(N)}  &  \multicolumn{2}{c}{CSI300}\\
 \cline{2-5}
  &  $m=5$  &  $m=10$  &$m=5$  & $m=10$    \\
 \hline
   $1.1$    &  $2.01\pm 0.49$  & $2.19\pm0.46$   &  $2.54\pm0.50$  & $2.75\pm0.43$   \\
   $1.2$    &  $2.00 \pm0.41 $  & $2.17\pm0.41 $   &  $2.38 \pm0.49 $  & $2.71 \pm0.46 $   \\
   $1.3$    &  $1.99 \pm0.31  $  & $2.15 \pm0.37 $   &  $2.19  \pm0.39  $  & $2.64\pm0.48 $   \\
   $1.4$    &  $1.99 \pm0.22$  & $2.12 \pm0.32 $   &  $2.03\pm0.17$  & $2.52\pm 0.50$   \\
   $1.5$    &  $1.99 \pm0.10$  & $2.07 \pm0.25$   &  $2.00  \pm0.00$  & $2.28 \pm 0.45$   \\
   $1.6$    &  $2.00\pm0.04$  & $2.03 \pm0.18$   &  $2.00  \pm0.00$  & $2.06\pm0.23 $   \\
   $1.7$    &  $2.00  \pm0.00$  & $2.00 \pm0.05$   &  $2.00  \pm0.00$  & $2.00  \pm0.00$   \\
   $1.8$    &  $2.00  \pm0.00$  & $2.00  \pm0.00$   &  $2.00  \pm0.00$  & $2.00  \pm0.00$   \\
   $1.9$    &  $2.00  \pm0.00$  & $2.00  \pm0.00$   &  $2.00  \pm0.00$  & $2.00  \pm0.00 $   \\
  \hline
\end{tabular}
}
}
\caption{Average number of iterates for $q$PWAWS to get away from the singular point (mean$\pm$STD).}
\label{tab:escapesing}
\vspace{-5pt}
\end{table}

Next, we need to verify that $q$PWAWS obtains the same $q$-th power median as $q$PWA if the latter does not get stuck. The experimental procedure is similar to the above. In each $q$-th power median computation, denote the $q$-th power medians of $q$PWAWS and $q$PWA by $\hat{\mathbf{x}}_{WAWS}$ and $\hat{\mathbf{x}}_{WA}$, respectively. Then we compute the maximum relative difference $\|\hat{\mathbf{x}}_{WAWS}-\hat{\mathbf{x}}_{WA}\|/\|\hat{\mathbf{x}}_{WA}\|$ of all the $(T-m+1)$ times of computations, shown in Table \ref{tab:sameresult}. Since the differences are close to zero ($<1e-07$), $q$PWAWS obtains nearly the same $q$-th power median as $q$PWA.

\begin{table}[!htb]
\vspace{-5pt}
\centering
\scalebox{0.9}{
\small{
\begin{tabular}{ccccc}
\hline
 \multirow{2}{*}{$q$} & \multicolumn{2}{c}{NYSE(N)}  &  \multicolumn{2}{c}{CSI300}\\
 \cline{2-5}
  &  $m=5$  &  $m=10$  &$m=5$  & $m=10$    \\
 \hline
   $1.1$    &  $5.9010e-09$  & $2.3797e-08$   &  $6.9425e-09$  & $2.4404e-08$   \\
   $1.2$    &  $6.3249e-09$  & $1.2411e-08$   &  $6.8934e-09$  & $6.1940e-09$   \\
   $1.3$    &  $5.4543e-09$  & $6.8011e-09$   &  $3.8610e-09$  & $4.0452e-09$   \\
   $1.4$    &  $5.2329e-09$  & $5.1516e-09$   &  $2.1728e-09$  & $3.9566e-09$   \\
   $1.5$    &  $4.5439e-09$  & $2.6621e-09$   &  $2.8953e-09$  & $2.4237e-09$   \\
   $1.6$    &  $1.8894e-09$  & $1.8003e-09$   &  $2.0482e-09$  & $1.4021e-09$   \\
   $1.7$    &  $1.4698e-09$  & $1.3701e-09$   &  $1.2259e-09$  & $1.0679e-09$   \\
   $1.8$    &  $6.1063e-10$  & $7.1172e-10$   &  $6.8329e-10$  & $4.7757e-10$   \\
   $1.9$    &  $2.1942e-10$  & $2.8647e-10$   &  $3.1570e-10$  & $3.0911e-10$   \\
  \hline
\end{tabular}
}
}
\caption{Maximum relative difference between $q$-th power medians of $q$PWAWS and $q$PWA: $\|\hat{\mathbf{x}}_{WAWS}-\hat{\mathbf{x}}_{WA}\|/\|\hat{\mathbf{x}}_{WA}\|$.}
\label{tab:sameresult}
\vspace{-5pt}
\end{table}

\subsection{Computational Costs and Convergence}
\label{sec:comcost}
A computer with an Intel Core i7-6700 CPU and a 4GB DDR3 memory card is used to record the computational time of $q$PWAWS. $(T-m+1)$ times of computations are recorded and the mean time cost for each $q$ and $m$ is shown in Table \ref{tab:comcost}. All the computations successfully converge and the average number of iterates for convergence (including the escaping iterates) is also shown in Table \ref{tab:comcost}. All the time costs ($<1e-03s$) and the numbers of iterates ($<27$) are small, which indicates that $q$PWAWS runs fast.

\begin{table}[!htb]
\vspace{-5pt}
\centering
\scalebox{0.93}{
\small{
\begin{tabular}{ccccc}
\hline
 \multirow{3}{*}{$q$} &  \multicolumn{2}{c}{$m=5$}  &  \multicolumn{2}{c}{$m=10$}     \\
 \cline{2-5} &  Time & Iters  &  Time &  Iters  \\
 \cline{2-5}  &  \multicolumn{4}{c}{NYSE(N)}   \\
 \hline
   $1.1$    &  $3.6059e-04$ & $25.31\pm 4.16$  & $6.0150e-04$ & $25.03\pm 6.50$\\
   $1.2$    &  $3.3310e-04$ & $22.78\pm 3.48$  & $5.5750e-04$ & $22.38\pm 4.64$     \\
   $1.3$    &  $3.0204e-04$ & $20.58\pm 3.06$  & $4.8230e-04$ & $20.08\pm 3.51$     \\
   $1.4$    &  $2.7411e-04$ & $18.51\pm 2.58$  & $4.4243e-04$ & $18.02\pm 2.70$      \\
   $1.5$    &  $2.5000e-04$ & $16.61\pm 2.12$  & $3.9745e-04$ & $16.12\pm 2.09$      \\
   $1.6$    &  $2.3567e-04$ & $14.85\pm 1.70$  & $3.5268e-04$ & $14.36\pm 1.60$      \\
   $1.7$    &  $2.0549e-04$ & $13.30\pm 1.20$  & $3.1274e-04$ & $12.71\pm 1.21$       \\
   $1.8$    &  $2.0305e-04$ & $11.71\pm 0.84$  & $2.7530e-04$ & $11.07\pm 0.91$       \\
   $1.9$    &  $1.5853e-04$ & $9.98\pm 0.55$  & $2.2369e-04$ & $9.31\pm 0.63$     \\
  \hline
   &  \multicolumn{4}{c}{CSI300} \\
   \hline
  $1.1$ & $5.0074e-04$ & $26.86\pm 4.85$  & $7.4105e-04$ & $26.89\pm 7.87$ \\
  $1.2$ & $3.9539e-04$ & $23.77\pm 3.70$  & $7.2515e-04$ & $23.97\pm 5.30$  \\
  $1.3$ & $4.2771e-04$ & $21.09\pm 3.19$  & $6.4174e-04$ & $21.49\pm 3.88$ \\
  $1.4$ & $3.8925e-04$ & $18.49\pm 2.74$  & $5.9630e-04$ & $19.18\pm 2.87$ \\
  $1.5$ & $3.2568e-04$ & $16.21\pm 2.31$  & $5.1977e-04$ & $16.85\pm 2.18$ \\
  $1.6$ & $3.2078e-04$ & $14.32\pm 2.09$  & $4.5260e-04$ & $14.82\pm 1.64$ \\
  $1.7$ & $3.2034e-04$ & $13.25\pm 1.39$  & $4.2355e-04$ & $13.06\pm 1.21$ \\
  $1.8$ & $2.8558e-04$ & $11.78\pm 0.96$  & $3.8196e-04$ & $11.33\pm 0.85$ \\
  $1.9$ &  $2.0650e-04$ & $10.06\pm 0.61$  & $3.2261e-04$ & $9.48\pm 0.56$ \\
  \hline
\end{tabular}}}
\caption{Average computational time (in seconds) and average number of iterates (mean$\pm$STD) for convergence of $q$PWAWS.}
\label{tab:comcost}
\vspace{-5pt}
\end{table}

\subsection{Rate of Convergence (Computational)}
\label{sec:converateexperiment}
In Section \ref{sec:converatespecial}, we have deduced the theoretical rate of convergence for $q$PWAWS in the special case $\mathbf{M}\in\{\mathbf{x}_i\}_{i=1}^m$. Now we conduct computational experiments to show that $q$PWAWS enjoys a reasonable rate of linear convergence in general. The procedure is the same as the above subsections and $(T-m+1)$ times of computations are implemented. In each computation, we take $\frac{\|\mathbf{y}_{(o-1)}-\mathbf{y}_{(o)}\|}{\|\mathbf{y}_{(o-2)}-\mathbf{y}_{(o)}\|}$ as an approximation for the rate of convergence, where $\mathbf{y}_{(o)}$ denotes the last iterate (i.e., the minimum) in Algorithm \ref{alg:lqwaws}. The means and the STDs of the rates of convergence for $q$PWAWS with different values of $q$ are shown in Table \ref{tab:converate}. When $q$ changes from $1$ to $1.9$, the empirical rate of convergence decreases from about $0.36$ to $0.06$, thus the newly-developed $1<q<2$ algorithm enjoys better rate of convergence than the long-used $q=1$ algorithm. Besides, all the situations with $1\leqs q<2$ achieve reasonable rates of linear convergence.

In order to see how the rate of convergence changes as $q$PWAWS runs, we plot the sequences of rates $\left\{\frac{\|\mathbf{y}_{(p+1)}-\mathbf{y}_{(o)}\|}{\|\mathbf{y}_{(p)}-\mathbf{y}_{(o)}\|}\right\}_{p=1}^{o-2}$ with $q=1.1\sim 1.9$ (one instance for each $q$) and $m=5$ in Figure \ref{fig:converate}. Each plot has a slowly-decreasing period and drops sharply in the last few iterates, thus the inflection point is close to the rate of convergence for $q$PWAWS, which suggests a linear convergence. Besides, the plots with smaller $q$s are above the plots with larger $q$s, which accords with the results in Table \ref{tab:converate}. Other instances of rates of convergence show similar patterns and we need not plot all of them.

\begin{table}[!htb]
\vspace{-5pt}
\centering
\scalebox{0.9}{
\small{
\begin{tabular}{ccccc}
\hline
 \multirow{2}{*}{$q$} & \multicolumn{2}{c}{NYSE(N)}  &  \multicolumn{2}{c}{CSI300}\\
 \cline{2-5}
  &  $m=5$  &  $m=10$  &$m=5$  & $m=10$    \\
  \hline
   $1$    &  $0.35\pm 0.03$  & $0.33\pm0.04$   &  $0.36\pm0.03$  & $0.34\pm0.04$   \\
 \hline
   $1.1$    &  $0.33\pm 0.03$  & $0.31\pm0.04$   &  $0.33\pm0.03$  & $0.32\pm0.04$   \\
   $1.2$    &  $0.31 \pm0.03 $  & $0.29\pm0.04 $   &  $0.31 \pm0.03 $  & $0.30 \pm0.04 $   \\
   $1.3$    &  $0.28\pm0.03  $  & $0.27 \pm0.04 $   &  $0.29  \pm0.03  $  & $0.27\pm0.04 $   \\
   $1.4$    &  $0.26 \pm0.03$  & $0.24 \pm0.04 $   &  $0.26\pm0.03$  & $0.25\pm 0.04$   \\
   $1.5$    &  $0.23 \pm0.03$  & $0.21 \pm0.03$   &  $0.23  \pm0.03$  & $0.22 \pm 0.03$   \\
   $1.6$    &  $0.19\pm0.03$  & $0.18 \pm0.03$   &  $0.19  \pm0.03$  & $0.18\pm0.03 $   \\
   $1.7$    &  $0.15  \pm0.02$  & $0.14 \pm0.02$   &  $0.15  \pm0.03$  & $0.15 \pm0.03$   \\
   $1.8$    &  $0.11  \pm0.02$  & $0.10  \pm0.02$   &  $0.11  \pm0.02$  & $0.10  \pm0.02$   \\
   $1.9$    &  $0.06  \pm0.01$  & $0.05  \pm0.01$   &  $0.06  \pm0.01$  & $0.05 \pm0.01 $   \\
  \hline
\end{tabular}
}
}
\caption{Average rate of convergence (mean$\pm$STD) for $q$PWAWS with different values of $q$.}
\label{tab:converate}
\vspace{-5pt}
\end{table}

\label{sec:plotrate}
\begin{figure*}[th]
\vspace{-5pt}
\centering
\subfloat[NYSE(N)]{
\centering
\includegraphics[width=0.45\textwidth]{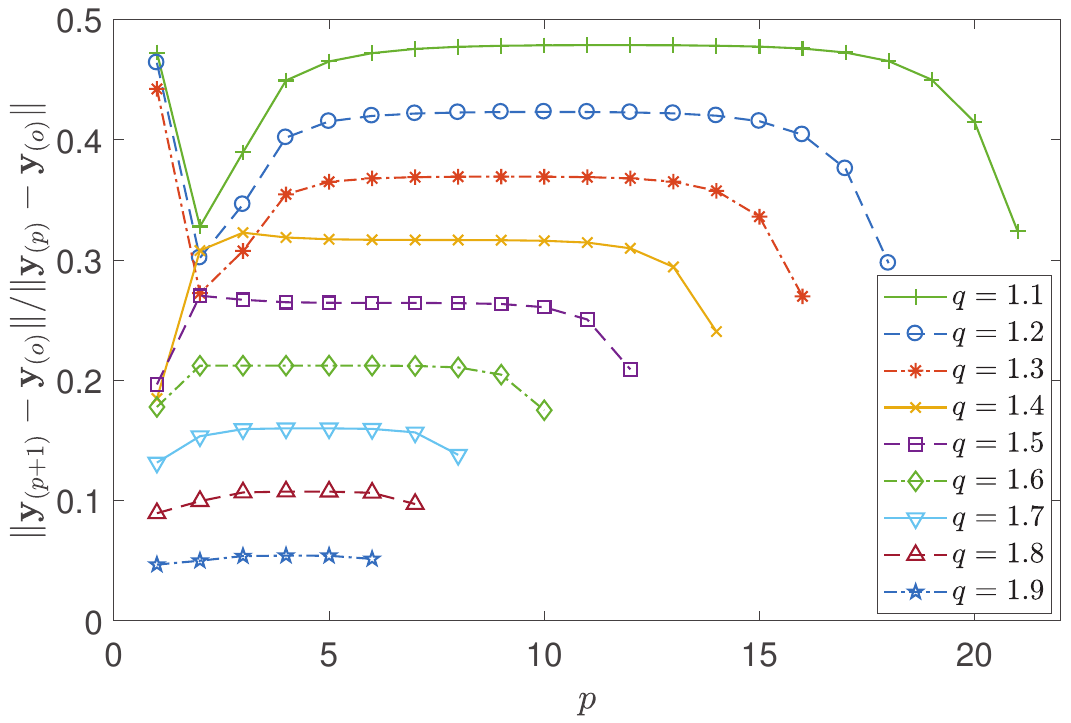}}
\subfloat[CSI300]{
\centering
\includegraphics[width=0.45\textwidth]{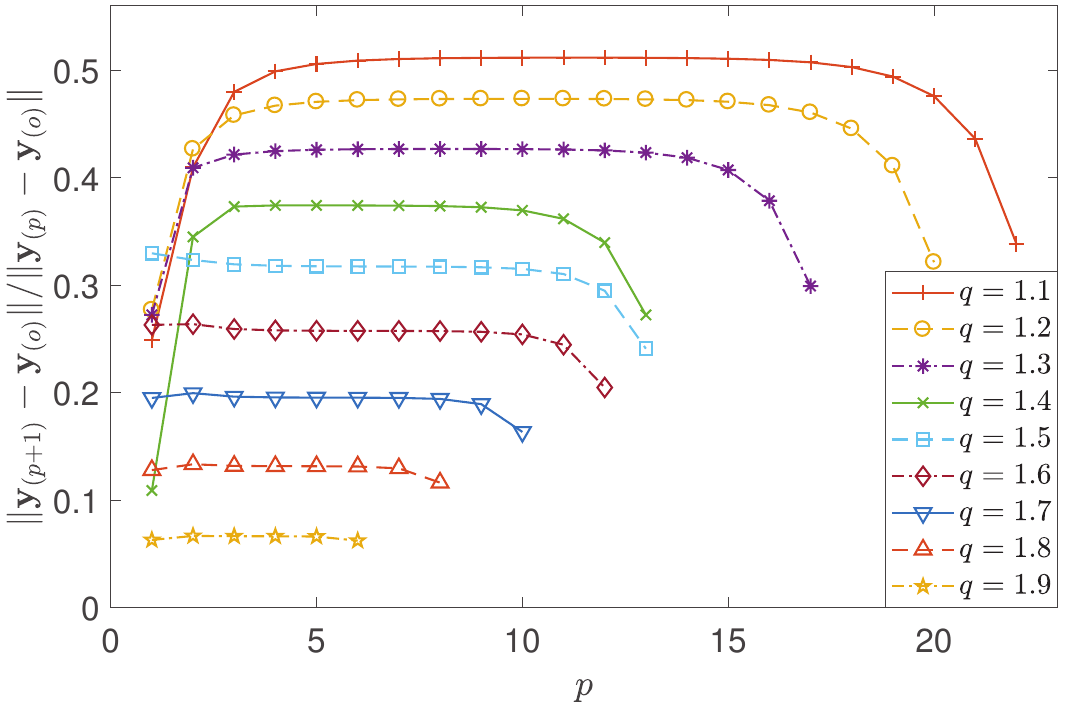}}
\caption{Sequences of rates of convergence $\left\{\frac{\|\mathbf{y}_{(p+1)}-\mathbf{y}_{(o)}\|}{\|\mathbf{y}_{(p)}-\mathbf{y}_{(o)}\|}\right\}_{p=1}^{o-2}$ for $q$PWAWS with $q=1.1\sim1.9$ and $m=5$.}
\label{fig:converate}
\vspace{-5pt}
\end{figure*}

\subsection{Investing Performance}
\label{sec:investperform}
We apply the $q$-th power median and $q$PWAWS to the OPS problem as a price prediction strategy and evaluate the investing performance on the two data sets NYSE(N) and CSI300. Interested readers are referred to \cite{RMR2} for the specific model setting. We change $q$ from $1$ to $2$ to see how it affects the investing performance. We set $m=5$ to be consistent with the window size of previous works.

The first indicator of evaluation is the final cumulative wealth (CW) when a strategy goes through the whole investment. The second indicator is the daily Sharpe Ratio (SR, \cite{SHARPratio}), which is a risk-adjusted average return that considers both return and risk in the investment. The results in Table \ref{tab:cwsr} indicate that $q=1.3$ and $q=1.6$ achieve the best CW and SR on NYSE(N) and CSI300, respectively. They improve on the trivial $q=1$ and $q=2$ to some extent, which suggests that the $q$-th power median ($1<q<2$) is more effective and robust than the $1$-st power median and the $2$-nd power median in some cases. Hence the $q$-th power median and $q$PWAWS are useful in this machine learning scenario.

\begin{table}[!htb]
\vspace{-5pt}
\centering
\scalebox{0.9}{
\small{
\begin{tabular}{ccccc}
\hline
 \multirow{2}{*}{$q$} & \multicolumn{2}{c}{NYSE(N)}  &  \multicolumn{2}{c}{CSI300}\\
 \cline{2-5}
  &  CW  &  SR  &   CW  & SR    \\
\hline
   $1$    &  $3.3183e+08$  & $0.1034$   &  $1.7750$  & $0.0479$   \\
\hline
   $1.1$    &  $1.0166e+09$  & $0.1082$   &  $1.7803$  & $0.0481$   \\
   $1.2$    &  $1.1564e+09$  & $0.1086$   &  $1.7544$  & $0.0473$   \\
   $1.3$    &  $\mathbf{1.1581e+09}$  & $\mathbf{0.1087}$   &  $1.7447$  & $0.0470$   \\
   $1.4$    &  $9.3994e+08$  & $0.1078$   &  $1.7469$  & $0.0471$   \\
   $1.5$    &  $8.4610e+08$  & $0.1074$   &  $1.8024$  & $0.0489$   \\
   $1.6$    &  $7.3675e+08$  & $0.1068$   &  $\mathbf{1.8434}$  & $\mathbf{0.0502}$   \\
   $1.7$    &  $6.6769e+08$  & $0.1063$   &  $1.8120$  & $0.0492$   \\
   $1.8$    &  $5.7382e+08$  & $0.1056$   &  $1.7892$  & $0.0485$   \\
   $1.9$    &  $4.7956e+08$  & $0.1047$   &  $1.7591$  & $0.0475$   \\
\hline
   $2$    &  $4.0764e+08$  & $0.1040$   &  $1.7311$  & $0.0466$   \\
  \hline
\end{tabular}
}
}
\caption{Cumulative wealth (CW) and Sharpe Ratio (SR) of $q$PWAWS with different values of $q$.}
\label{tab:cwsr}
\vspace{-5pt}
\end{table}

\section{Conclusions and Future Works}
\label{sec:conclusion}
This paper mainly establishes a novel de-singularity subgradient approach and a corresponding algorithm ($q$-th Power Weiszfeld Algorithm without Singularity, $q$PWAWS) for the extended Weber location problem. We characterize the subgradient(s) and the optimality of any given singular point. If this singular point is not optimal, the algorithm can escape from it and reduce the cost simultaneously. This advantage makes the sequence of $q$PWAWS monotonically converge to the exact minimum. A complete proof of convergence for $q$PWAWS is also presented, which has fixed some incomplete statements of the proofs for some previous Weiszfeld algorithms. $q$PWAWS enjoys a superlinear convergence for the iteration sequence in the special case where the minimum point is a singular point, which is a new theoretical result.

Experiments with real-world financial data sets indicate that $q$PWAWS successfully gets out of the singular point with only a small number of iterates. Besides, $q$PWAWS runs fast, converges with only a few iterates, obtains the same $q$-th power median as $q$PWA if the latter does not get stuck, and shows a reasonable rate of linear convergence. In some cases of the online portfolio selection, the $q$-th power median ($1<q<2$) outperforms the $1$-st power median and the $2$-nd power median in both the final cumulative wealth and the Sharpe Ratio, which suggests that the $q$-th power median ($1<q<2$) achieves better investing performance. Therefore, the de-singularity subgradient approach is beneficial to advancing both theory and practice for the extended Weber location problem.

Future works may fall into the following aspects: 1. Find the optimal de-singularity subgradient descent in (\ref{eqn:singnextiter}). 2. Establish the de-singularity subgradient theory and algorithm for the $q$-th power and $L_p$-norm median problem. 3. Extend the de-singularity subgradient theory to other related optimization problems that are also challenged by the singularity problem.

\section*{Acknowledgments}
This work is supported in part by the National Natural Science Foundation of China under grants 62176103, 62276114, 62206110, in part by the Science and Technology Planning Project of Guangzhou under grants 2024A04J9896, 202206030007, Nansha District: 2023ZD001, in part by Guangdong-Macao Advanced Intelligent Computing Joint Laboratory under grant 2020B1212030003, in part by the Key Laboratory of Smart Education of Guangdong Higher Education Institutes, Jinan University under grant 2022LSYS003, and in part by the Fundamental Research Funds for the Central Universities, JNU under grant 21623202.

\bibliographystyle{named}
\bibliography{qrref}

\clearpage

\appendix
\onecolumn

\section*{Supplementary Material}
\vspace{10pt}

\section{Additional Content}

\subsection{Update Formula with Coincidence for $q=1$}
\label{sec:updatecoinqone}
The $q=1$ case is solved by \cite{l1median3} and interested readers are referred to it for a detailed proof. We directly give the $q=1$ formula as follows:
\begin{align}
\label{eqn:l1wawsgen}
\tilde{\mathbf{T}}(\mathbf{y}_{(p)})&= \frac{\sum_{i\ne k} \eta_i\| \mathbf{y}_{(p)}-\mathbf{x}_i  \|^{-1}\mathbf{x}_i}{\sum_{i\ne k} \eta_i\| \mathbf{y}_{(p)}-\mathbf{x}_i  \|^{-1}},\\
\label{eqn:l1wawsgen2}
\mathbf{y}_{(p+1)}&= (1-\lambda)\tilde{\mathbf{T}}(\mathbf{y}_{(p)})+\lambda \mathbf{y}_{(p)}, \nonumber\\ 
\lambda &= \min \left\{ 1, \frac{\eta_k}{\|\nabla D_1(\mathbf{y}_{(p)})\|} \right\}.\qquad
\end{align}
This strategy ensures $\mathbf{y}_{(p+1)}\ne \mathbf{y}_{(p)} \Leftrightarrow$ $C_1(\mathbf{y}_{(p+1)})<C_1(\mathbf{y}_{(p)})$ and $\mathbf{y}_{(p+1)}= \mathbf{y}_{(p)} \Leftrightarrow \mathbf{y}_{(p)}=\mathbf{M}$.

\subsection{Solving Algorithm}
\label{sec:solvealgo}
Note that the multiplicities are changed from $\{\eta_i\}_{i=1}^m$ back to $\{\xi_i\}_{i=1}^m$ to simplify the expressions.
\small
\begin{breakablealgorithm}
\caption{$q$-th power Weiszfeld algorithm without singularity ($q$PWAWS)}
\label{alg:lqwaws}
\begin{algorithmic}
\REQUIRE Given $m$ distinct data points $\{\mathbf{x}_i\}_{i=1}^m$, the corresponding multiplicities $\{\xi_i\}_{i=1}^m$, the order of power $q$, the reducing factor $\rho$ and the tolerance threshold $Tol$. \\
\STATE 1. Initialize with a starting point $\mathbf{y}_{(0)}$.
\WHILE{1}
\IF {$\mathbf{y}_{(p)}\notin \{\mathbf{x}_i\}_{i=1}^m$}
\STATE 2. Compute $ \mathbf{y}_{(p+1)}=\frac{\sum_{i=1}^m \xi_i\| \mathbf{y}_{(p)}-\mathbf{x}_i  \|^{q-2}\mathbf
{x}_i}{\sum_{i=1}^m \xi_i\| \mathbf{y}_{(p)}-\mathbf{x}_i  \|^{q-2}}$.
\IF {$\mathbf{y}_{(p+1)}=\mathbf{y}_{(p)}$}
\STATE 3. $\mathbf{M}=\mathbf{y}_{(p)}$. Break.
\ENDIF
\STATE 4. $p\leftarrow p+1$.
\ELSE
\STATE 5. Suppose $\mathbf{y}_{(p)}=\mathbf{x}_k$, 
\STATE \quad compute $\nabla D_q(\mathbf{x}_k)=\sum_{i\ne k} q\xi_i\| \mathbf{x}_k-
\mathbf{x}_i  \|^{q-2}(\mathbf{x}_k-\mathbf{x}_i)$.
\IF {$q=1$}
\IF {$\|\nabla D_1(\mathbf{x}_k)\|\leqs \xi_k$}
\STATE 6. $\mathbf{M}=\mathbf{x}_k$. Break.
\ELSE
\STATE 7. Compute $\tilde{\mathbf{T}}(\mathbf{x}_k)=\frac{\sum_{i\ne k} \xi_i\| \mathbf{x}_k-\mathbf{x}_i  \|^{-1}\mathbf
{x}_i}{\sum_{i\ne k} \xi_i\| \mathbf{x}_k-\mathbf{x}_i  \|^{-1}}$, 
\STATE $\lambda = \frac{\xi_k}{\|\nabla D_1(\mathbf{x}_k)\|}$,
$\mathbf{y}_{(p+1)}=(1-\lambda)\tilde{\mathbf{T}}(\mathbf{x}_k)+\lambda \mathbf{x}_k$.
\STATE 8. $p\leftarrow p+1$.
\ENDIF
\ELSE
\IF {$\|\nabla D_q(\mathbf{x}_k)\|=0$}
\STATE 9. $\mathbf{M}=\mathbf{x}_k$. Break.
\ELSE
\STATE 10. Set $w=0$, 
\STATE $\lambda_w= \min\left\{\frac{1}{q}\xi_k^{-\frac{1}{q-1}}\| \nabla D_q(\mathbf{x}_k) \|^{\frac{2-q}{q-1}},1\right\}$.
\WHILE {$C_q(\mathbf{x}_k-\lambda_w \nabla D_q(\mathbf{x}_k))\geqs C_q(\mathbf{x}_k)$}
\STATE 11. $\lambda_{w+1}=\rho\lambda_w$. $w\leftarrow w+1$.
\ENDWHILE
\STATE 12. $\mathbf{y}_{(p+1)}=\mathbf{x}_k-\lambda_w \nabla D_q(\mathbf{x}_k)$. $p\leftarrow p+1$.
\ENDIF
\ENDIF
\ENDIF
\IF {$\|\mathbf{y}_{(p+1)}-\mathbf{y}_{(p)}\|/\|\mathbf{y}_{(p)}\|\leqs Tol$}
\STATE 13. $\mathbf{M}=\mathbf{y}_{(p+1)}$. Break.
\ENDIF
\ENDWHILE
\ENSURE The minimum point $\mathbf{M}$.
\end{algorithmic}
\end{breakablealgorithm}
\normalsize

\section{Proofs}

\subsection{Proof of Theorem \ref{thm:nonincreasing}}
\label{proof:nonincreasing}
To prove this theorem, we need the following lemma:
\begin{lemma}[\cite{l1median2,extendfermat,extendfermat2,lqmean}]
\label{lem:nonincreasing}
If $a_i>0$ and $b_i>0$, $0<q<n$ and $\sum_{i=1}^m a_i^{q-n}b_i^n< \sum_{i=1}^m a_i^{q}$, then $\sum_{i=1}^m b_i^{q}\leqs\sum_{i=1}^m a_i^{q}$ and the equality holds only when $a_i=b_i, \forall i$.
\end{lemma}

\begin{proof}
Consider the following function $g(t)$:
\begin{align}
\label{eqn:strickcont}
g(t)=\sum_{i=1}^m a_i^{q-t}b_i^t, 0\leqs t\leqs n.
\end{align}
The second derivative of $g$ with respect to $t$ is:
\begin{align}
\label{eqn:strickconsecond}
g''(t)=\sum_{i=1}^m a_i^{q-t}b_i^t(\log a_i-\log b_i)^2.
\end{align}
Since all the $a_i,b_i>0$, then $g''(t)>0$ and $g(t)$ is a strictly convex function unless $a_i=b_i, \forall i$. If $g(t)$ is a strictly convex function, then $g(n)<g(0)$ implies $g(q)<g(0)$. Thus the lemma is proven.  
\end{proof}

Lemma \ref{lem:nonincreasing} reveals the relation between the $q$-th power ($1\leqs q<2$) cost in (\ref{eqn:lqmediangenreal}) and the following weighted $2$-nd power cost:
\begin{align}
\label{eqn:approxcqgen}
\tilde{C}_q(\mathbf{y})= \sum_{i=1}^m \eta_i^q\| \mathbf{y}_{(p)}-\mathbf{x}_i  \|^{q-2}\| \mathbf{y}-\mathbf{x}_i\|^2.
\end{align}

\begin{proof}
$\tilde{C}_q(\mathbf{y})$ in (\ref{eqn:approxcqgen}) is a strictly convex function on $\mathbf{y}$. By taking the gradient of $\tilde{C}_q(\mathbf{y})$ and setting it to zero, it yields:
\begin{align}
\label{eqn:approxcqgengrad}
\nabla\tilde{C}_q(\mathbf{y})= \sum_{i=1}^m 2\eta_i^q\| \mathbf{y}_{(p)}-\mathbf{x}_i  \|^{q-2}( \mathbf{y}-\mathbf{x}_i)=\mathbf{0}.
\end{align}
Hence $\mathbf{T}_1( \mathbf{y}_{(p)})$ is the minimizer of $\tilde{C}_q(\mathbf{y})$. It yields $\tilde{C}_q(\mathbf{T}_1( \mathbf{y}_{(p)}))\leqs \tilde{C}_q(\mathbf{y}_{(p)})= C_q(\mathbf{y}_{(p)})$ with equality holds only when $\mathbf{T}_1( \mathbf{y}_{(p)})= \mathbf{y}_{(p)}$.

If $\tilde{C}_q(\mathbf{T}_1( \mathbf{y}_{(p)}))<C_q(\mathbf{y}_{(p)})$, it means:
\begin{align}
\label{eqn:approxineq}
\sum_{i=1}^m \eta_i^q\| \mathbf{y}_{(p)}-\mathbf{x}_i  \|^{q-2}\| \mathbf{T}_1( \mathbf{y}_{(p)})-\mathbf{x}_i\|^2 &<\sum_{i=1}^m \eta_i^q\| \mathbf{y}_{(p)}-\mathbf{x}_i  \|^{q}, \nonumber\\
\sum_{i=1}^m \| \eta_i(\mathbf{y}_{(p)}{-}\mathbf{x}_i)  \|^{q-2}\| \eta_i(\mathbf{T}_1( \mathbf{y}_{(p)}){-}\mathbf{x}_i)\|^2 &<\sum_{i=1}^m \| \eta_i(\mathbf{y}_{(p)}{-}\mathbf{x}_i)  \|^{q}.
\end{align}
By setting $a_i=\| \eta_i(\mathbf{y}_{(p)}{-}\mathbf{x}_i)  \|$, $b_i=\| \eta_i(\mathbf{T}_1( \mathbf{y}_{(p)}){-}\mathbf{x}_i)\|$ for all $i$ and using Lemma \ref{lem:nonincreasing} ($n=2$), it leads to:
\begin{align}
\label{eqn:approxineq2}
C_q(\mathbf{T}_1( \mathbf{y}_{(p)})) =&\sum_{i=1}^m \| \eta_i(\mathbf{T}_1( \mathbf{y}_{(p)}){-}\mathbf{x}_i)\|^q<\sum_{i=1}^m \| \eta_i(\mathbf{y}_{(p)}{-}\mathbf{x}_i)  \|^{q} =C_q(\mathbf{y}_{(p)}).
\end{align}
It proves Theorem \ref{thm:nonincreasing}.  
\end{proof}

\subsection{Proof of Corollary \ref{cor:charnonsing}}
\label{proof:charnonsing}

\begin{proof}
With $\mathbf{y}_{(p)}\notin \{\mathbf{x}_i\}_{i=1}^m$ and (\ref{eqn:lqwa3}), the following equivalence holds:
\begin{align}
\label{eqn:lqwa4}
&\mathbf{T}_1( \mathbf{y}_{(p)})= \mathbf{y}_{(p)} \quad\Longleftrightarrow \quad   \mathbf{0}= \sum_{i=1}^m \eta_i^q\| \mathbf{y}_{(p)}-\mathbf{x}_i  \|^{q-2}(\mathbf{y}_{(p)}-\mathbf{x}_i)=\frac{1}{q}\nabla C_q(\mathbf{y}_{(p)}).
\end{align}
Since $C_q(\mathbf{y})$ is strictly convex, $\nabla C_q(\mathbf{y}_{(p)})=\mathbf{0} \Leftrightarrow \mathbf{y}_{(p)}=\mathbf{M}$.  
\end{proof}

\subsection{Proof of Theorem \ref{thm:charsing}}
\label{proof:charsing}

\begin{proof}
Let $\mathbf{x}_k+\lambda \mathbf{z}$ $(\lambda>0, \|\mathbf{z}\|=1)$ be a point displaced from $\mathbf{x}_k$ towards an arbitrary direction. Then the gradient of $C_q(\mathbf{x}_k+\lambda \mathbf{z})$ with respect to $\lambda$ is:
\begin{align}
\label{eqn:lqgradlbd}
&\frac{\ud C_q(\mathbf{x}_k+\lambda \mathbf{z})}{\ud\lambda}{=} \sum_{i\ne k} q\eta_i^q\| \mathbf{x}_k+\lambda \mathbf{z}-\mathbf{x}_i  \|^{q-2}(\mathbf{x}_k+\lambda \mathbf{z}-\mathbf{x}_i)^\top \mathbf{z} {+}q\eta_k^q\lambda^{q-1}.
\end{align}
The limit of $\frac{\ud}{\ud\lambda}C_q(\mathbf{x}_k+\lambda \mathbf{z})$ when $\lambda \rightarrow 0$ is:
\begin{subequations}
\begin{align}
\label{eqn:lqgradlbd01}
&\frac{\ud C_1(\mathbf{x}_k+\lambda \mathbf{z})}{\ud\lambda}\arrowvert_{\lambda=0}= \sum_{i\ne k} \eta_i\| \mathbf{x}_k-\mathbf{x}_i  \|^{-1}(\mathbf{x}_k-\mathbf{x}_i)^\top \mathbf{z}+\eta_k, \quad q=1.\\
\label{eqn:lqgradlbd0q}
&\frac{\ud C_q(\mathbf{x}_k+\lambda \mathbf{z})}{\ud\lambda}\arrowvert_{\lambda=0}= \sum_{i\ne k} q\eta_i^q\| \mathbf{x}_k-\mathbf{x}_i  \|^{q-2}(\mathbf{x}_k-\mathbf{x}_i)^\top \mathbf{z},\quad 1<q<2.
\end{align}
\end{subequations}
From Definition \ref{defn:desinggrad}, (\ref{eqn:lqgradlbd01}) and (\ref{eqn:lqgradlbd0q}) can be formulated as:
\begin{subequations}
\begin{align}
\label{eqn:lqgradlbd01var}
\frac{\ud C_1(\mathbf{x}_k+\lambda \mathbf{z})}{\ud\lambda}\arrowvert_{\lambda=0}&=\nabla D_1(\mathbf{x}_k)^\top\mathbf{z}+\eta_k, \quad q=1.\\
\label{eqn:lqgradlbd0qvar}
\frac{\ud C_q(\mathbf{x}_k+\lambda \mathbf{z})}{\ud\lambda}\arrowvert_{\lambda=0}&=\nabla D_q(\mathbf{x}_k)^\top\mathbf{z}, \quad 1<q<2.
\end{align}
\end{subequations}
Thus the multiplicity $\eta_k$ affects the gradient only when $q=1$. By setting \\ $\displaystyle{\mathbf{z}=-\frac{\nabla D_q(\mathbf{x}_k)}{\|\nabla D_q(\mathbf{x}_k)\|}}$ in (\ref{eqn:lqgradlbd01var}) and (\ref{eqn:lqgradlbd0qvar}), we have:
\begin{subequations}
\begin{align}
\label{eqn:lqgradlbd01varmin}
&\min_{\mathbf{z}}\frac{\ud C_1(\mathbf{x}_k+\lambda \mathbf{z})}{\ud\lambda}\arrowvert_{\lambda=0}=-\|\nabla D_1(\mathbf{x}_k)\|+\eta_k,\quad  q=1.\qquad\\
\label{eqn:lqgradlbd0qvarmin}
&\min_{\mathbf{z}}\frac{\ud C_q(\mathbf{x}_k+\lambda \mathbf{z})}{\ud\lambda}\arrowvert_{\lambda=0}=-\|\nabla D_q(\mathbf{x}_k)\|,\quad  1<q<2.\\
\label{eqn:lqminchar2}
&C_q(\mathbf{x}_k) \:\text{is the minimum} \quad\Longleftrightarrow\quad  \min_{\mathbf{z}}\frac{\ud C_q(\mathbf{x}_k+\lambda \mathbf{z})}{\ud\lambda}\arrowvert_{\lambda=0}\geqs 0, \: 1\leqs q<2.
\end{align}
\end{subequations}
Combining (\ref{eqn:lqgradlbd01varmin}), (\ref{eqn:lqgradlbd0qvarmin}) and (\ref{eqn:lqminchar2}), one can find that the subgradient sets in (\ref{eqn:subgrad}) are equivalent to Definition \ref{def:frechetsubdifferential}. Thus Theorem \ref{thm:charsing} is proven.  
\end{proof}

\subsection{Proof of Theorem \ref{thm:singmin}}
\label{proof:singmin}

\begin{proof}
This theorem indicates that a sufficiently small displacement towards the negative de-singularity subgradient $-\nabla D_q(\mathbf{x}_k)$ can reduce the $q$-th power cost. $C_q(\mathbf{x}_k-\lambda \nabla D_q(\mathbf{x}_k))$ is continuous on $\lambda>0$. It consists of two parts: the nonsingular part $D_q(\mathbf{x}_k-\lambda \nabla D_q(\mathbf{x}_k))$ and the singular part $\eta_k^q\lambda^q\| \nabla D_q(\mathbf{x}_k)  \|^q$. From the Taylor series expansion of $D_q(\mathbf{x}_k-\lambda \nabla D_q(\mathbf{x}_k))$,
\begin{align}
\label{eqn:cqdecom}
&C_q(\mathbf{x}_k{-}\lambda \nabla D_q(\mathbf{x}_k))\nonumber\\
= &D_q(\mathbf{x}_k{-}\lambda \nabla D_q(\mathbf{x}_k)){+}\eta_k^q\lambda^q\| \nabla D_q(\mathbf{x}_k)  \|^q \nonumber\\
= &D_q(\mathbf{x}_k)-\lambda \| \nabla D_q(\mathbf{x}_k) \|^2+\frac{\lambda^2}{2}\nabla D_q(\mathbf{x}_k)^\top H(\mathbf{x}_k)\nabla D_q(\mathbf{x}_k)+o(\lambda^2)+\eta_k^q\lambda^q\| \nabla D_q(\mathbf{x}_k)  \|^q,
\end{align}
where $H(\mathbf{x}_k)$ is the Hessian of $D_q(\mathbf{y})$ at $\mathbf{x}_k$. Besides, it is easy to find that $D_q(\mathbf{x}_k)=C_q(\mathbf{x}_k)$. Then (\ref{eqn:cqdecom}) can be rearranged to:
\begin{align}
\label{eqn:cqdecom2}
&C_q(\mathbf{x}_k{-}\lambda \nabla D_q(\mathbf{x}_k)){-}C_q(\mathbf{x}_k) = -\lambda \| \nabla D_q(\mathbf{x}_k) \|^2{+}\frac{\lambda^2}{2}G(\mathbf{x}_k)+o(\lambda^2){+}\eta_k^q\lambda^q\| \nabla D_q(\mathbf{x}_k)  \|^q,
\end{align}
where $G(\mathbf{x}_k)\triangleq \nabla D_q(\mathbf{x}_k)^\top H(\mathbf{x}_k)\nabla D_q(\mathbf{x}_k)$. Therefore, we need to find a $\lambda$ such that the right side of (\ref{eqn:cqdecom2}) is negative.

When $\lambda \rightarrow 0$, the negative term $-\lambda \| \nabla D_q(\mathbf{x}_k) \|^2$ dominates the other terms on the right side, thus it is possible to make the right side negative. To specify, a $\lambda$ should be found to satisfy the following inequality:
\begin{align}
\label{eqn:cqdecom3}
-\lambda \| \nabla D_q(\mathbf{x}_k) \|^2&{+}\frac{\lambda^2}{2}G(\mathbf{x}_k){+}o(\lambda^2){+}\eta_k^q\lambda^q\| \nabla D_q(\mathbf{x}_k)  \|^q<0.\quad
\end{align}
Dividing both sides of (\ref{eqn:cqdecom3}) by $\lambda$ yields:
\begin{align}
\label{eqn:cqdecom4}
&{-}\| \nabla D_q(\mathbf{x}_k) \|^2 {+}\frac{\lambda}{2}G(\mathbf{x}_k){+}o(\lambda){+}\eta_k^q\lambda^{q-1}\| \nabla D_q(\mathbf{x}_k)  \|^q<0,\quad \nonumber\\
&\frac{\lambda}{2}G(\mathbf{x}_k)+o(\lambda)+\eta_k^q\lambda^{q-1}\| \nabla D_q(\mathbf{x}_k)  \|^q< \| \nabla D_q(\mathbf{x}_k) \|^2.
\end{align}
Since $1<q<2$, the left side of (\ref{eqn:cqdecom4}) approaches zero when $\lambda \rightarrow 0$, while $\| \nabla D_q(\mathbf{x}_k) \|^2>0$. Therefore, there exists a $\lambda_*>0$ such that for any $0<\lambda\leqs\lambda_*$, (\ref{eqn:cqdecom4}) holds. From (\ref{eqn:cqdecom4}) back to (\ref{eqn:cqdecom2}), the theorem is proven.  
\end{proof}

\subsection{Proof of Lemma \ref{lem:convexhull}}
\label{proof:convexhull}

\begin{proof}
From Corollary \ref{cor:charnonsing}, Theorem \ref{thm:nonincreasing}, Theorem \ref{thm:charsing} and Theorem \ref{thm:singmin}, $q$PWAWS has the following decreasing property:
\begin{align}
\label{eqn:lqwawsdecrease}
C_q(\mathbf{y}_{(0)})>C_q(\mathbf{y}_{(1)})>\cdots >C_q(\mathbf{y}_{(p)})>\cdots>C_q(\mathbf{M}),
\end{align}
unless some $\mathbf{y}_{(p)}$ hits $\mathbf{M}$. In particular, if $\mathbf{y}_{(p)}=\mathbf{x}_k$ but $\mathbf{x}_k\ne \mathbf{M}$, then $C_q(\mathbf{y}_{(p)})>C_q(\mathbf{y}_{(p+1)})$ and the subsequent iterates will never get back to $\mathbf{x}_k$, otherwise the decreasing property will be violated. Hence the sequence of iterates visits each $\mathbf{x}_k\ne \mathbf{M}$ at most once and will not get stuck.

From (\ref{eqn:lqwagen}), if $\mathbf{y}_{(p)}\notin \{\mathbf{x}_i\}_{i=1}^m$, $\mathbf{T}_1( \mathbf{y}_{(p)})$ is a weighted sum of the data points $\{\mathbf{x}_i\}_{i=1}^m$ with positive weights that sum to one. Hence $\mathbf{y}_{(p+1)}{=}\mathbf{T}_1( \mathbf{y}_{(p)})$ lies in the convex hull of $\{\mathbf{x}_i\}_{i=1}^m$. Since $\{\mathbf{y}_{(p)}\}$ visits each $\mathbf{x}_k\ne \mathbf{M}$ at most once, $\mathbf{T}_2( \mathbf{y}_{(p)})$ is invoked at most finite times. Because $\mathbf{T}_2( \mathbf{y}_{(p)})$ cannot ensure that $\mathbf{y}_{(p+1)}$ lies in the convex hull, there are at most a finite set of iterates that do not lie in the convex hull.

Last, if $\mathbf{M}\in \{\mathbf{x}_i\}_{i=1}^m$, then $\mathbf{M}$ is trivially in the convex hull. If $\mathbf{M}\notin \{\mathbf{x}_i\}_{i=1}^m$, Corollary \ref{cor:charnonsing} and (\ref{eqn:lqwa4}) indicate that $\mathbf{M}$ lies in the convex hull.  
\end{proof}

\subsection{Proof of Lemma \ref{lem:contintheorem}}
\label{proof:contintheorem}

\begin{proof}
First, taking a difference between both sides of (\ref{eqn:continlem}) leads to
\begin{align}
\label{eqn:contintheoremproof}
\mathbf{T}_1(\mathbf{y})-\mathbf{x}_k =&\frac{\sum_{i\ne k} \eta_i^q\| \mathbf{y}-\mathbf{x}_i  \|^{q-2}(\mathbf{x}_i-\mathbf{x}_k)}{\sum_{i=1}^m \eta_i^q\| \mathbf{y}-\mathbf{x}_i  \|^{q-2}}\\
\label{eqn:contintheoremproof2}
=&\frac{\| \mathbf{y}-\mathbf{x}_k  \|^{2-q}\cdot(\sum_{i\ne k} \eta_i^q\| \mathbf{y}-\mathbf{x}_i  \|^{q-2}(\mathbf{x}_i-\mathbf{x}_k))}{\eta_k^q+\| \mathbf{y}-\mathbf{x}_k  \|^{2-q}\cdot(\sum_{i\ne k} \eta_i^q\| \mathbf{y}-\mathbf{x}_i  \|^{q-2})}.
\end{align}
Then the limit of its $L_2$-norm is
\begin{align}
\label{eqn:contintheoremproof3}
&\lim_{\mathbf{y}\rightarrow \mathbf{x}_k}\|\mathbf{T}_1(\mathbf{y}){-}\mathbf{x}_k\| {=}\frac{0\cdot \|\sum_{i\ne k} \eta_i^q\| \mathbf{x}_k-\mathbf{x}_i  \|^{q-2}(\mathbf{x}_i-\mathbf{x}_k)\|}{\eta_k^q+0\cdot (\sum_{i\ne k} \eta_i^q\| \mathbf{x}_k-\mathbf{x}_i  \|^{q-2})}=0.
\end{align}
Since $\eta_k^q\ne 0$, the above limit is well-defined and equals $0$. It indicates that $\mathbf{T}_1(\mathbf{y})\rightarrow \mathbf{x}_k$ when $\mathbf{y}\rightarrow \mathbf{x}_k$.  
\end{proof}

\subsection{Proof of Lemma \ref{lem:kickout}}
\label{proof:kickout}

\begin{proof}
From (\ref{thm:charsing}), if $\mathbf{x}_k\ne \mathbf{M}$, then $\|\nabla D_q(\mathbf{x}_k)\|>0$. This is the key condition for $\mathbf{T}_1$ to drive $\mathbf{y}$ away. For any sufficiently small $0<\epsilon<1$, since $\|\nabla D_q(\mathbf{y})\|$ is continuous around $\mathbf{x}_k$, there exists $\delta_1>0$ such that
\begin{align}
\label{eqn:delta1choose}
\mathbf{y}\in B(\mathbf{x}_k,\delta_1) \Longrightarrow \|\nabla D_q(\mathbf{y})\|>\|\nabla D_q(\mathbf{x}_k)\|-\epsilon>0.
\end{align}

Second, when $\mathbf{y}\rightarrow \mathbf{x}_k$, the weight of $\mathbf{x}_k$ in (\ref{eqn:lqwa3}) will approach $1$. In other words, there exists $\delta_2>0$ such that
\begin{align}
\label{eqn:delta2choose}
\mathbf{y}\in B(\mathbf{x}_k,\delta_2) \Longrightarrow  1-\epsilon<\frac{ \eta_k^q\| \mathbf{y}-\mathbf{x}_k  \|^{q-2}}{\sum_{i=1}^m \eta_i^q\| \mathbf{y}-\mathbf{x}_i  \|^{q-2}}< 1.
\end{align}

Third, to handle some remainders, define $\delta_3>0$ as follows:
\begin{align}
\label{eqn:delta3choose}
& \delta_3=\left(  \frac{(1-\epsilon)(\|\nabla D_q(\mathbf{x}_k)\|-\epsilon)}{q\eta_k^q(1+2\epsilon)}   \right)^{1/(q-1)}, \quad 1<q<2.\\
 \label{eqn:delta3choose2}
& \mathbf{y}\in B(\mathbf{x}_k,\delta_3)  \Longrightarrow  \frac{(1-\epsilon)(\|\nabla D_q(\mathbf{x}_k)\|-\epsilon)}{q\eta_k^q\| \mathbf{y}-\mathbf{x}_k  \|^{q-1}}>1+2\epsilon.
\end{align}
When $0<\epsilon<1$ is sufficiently small, $\delta_3>0$ is well defined.

Let $\delta_0=\min\{\delta_1,\delta_2,\delta_3\}$ and $\mathbf{y}\in B(\mathbf{x}_k,\delta_0)$, then:
\begin{align}
\label{eqn:t1xkdiff}
\mathbf{T}_1(\mathbf{y})-\mathbf{x}_k =&\frac{\sum_{i=1}^m \eta_i^q\| \mathbf{y}-\mathbf{x}_i  \|^{q-2}(\mathbf{x}_i-\mathbf{y})}{\sum_{i=1}^m \eta_i^q\| \mathbf{y}-\mathbf{x}_i  \|^{q-2}}+\mathbf{y}- \mathbf{x}_k\nonumber\\
=&\frac{-\nabla D_q(\mathbf{y})/q}{\sum_{i=1}^m \eta_i^q\| \mathbf{y}{-}\mathbf{x}_i  \|^{q-2}} {+}\left(  \frac{ \eta_k^q\| \mathbf{y}{-}\mathbf{x}_k  \|^{q-2}}{\sum_{i=1}^m \eta_i^q\| \mathbf{y}{-}\mathbf{x}_i  \|^{q-2}}{-}1 \right)(\mathbf{x}_k{-}\mathbf{y}).
\end{align}
Therefore
\begin{align}
\label{eqn:t1xkdiff2}
 \|\mathbf{T}_1(\mathbf{y})-\mathbf{x}_k\|  >&\frac{\|\nabla D_q(\mathbf{y})/q\|}{\sum_{i=1}^m \eta_i^q\| \mathbf{y}-\mathbf{x}_i  \|^{q-2}}-\epsilon\|\mathbf{x}_k-\mathbf{y}\|\nonumber\\
>&\frac{(1-\epsilon)\|\nabla D_q(\mathbf{y})/q\|}{\eta_k^q\| \mathbf{y}-\mathbf{x}_k  \|^{q-2}}-\epsilon\|\mathbf{x}_k-\mathbf{y}\|\nonumber\\
>&\frac{(1-\epsilon)(\|\nabla D_q(\mathbf{x}_k)\|-\epsilon)}{q\eta_k^q\| \mathbf{y}-\mathbf{x}_k  \|^{q-2}}-\epsilon\|\mathbf{x}_k-\mathbf{y}\|\nonumber\\
>&(1+2\epsilon)\|\mathbf{x}_k-\mathbf{y}\|-\epsilon\|\mathbf{x}_k-\mathbf{y}\|\nonumber\\
=&(1+\epsilon)\|\mathbf{x}_k-\mathbf{y}\|,
\end{align}
where the first inequality is based on the triangle inequality and (\ref{eqn:delta2choose}); The second inequality is based on the left inequality of (\ref{eqn:delta2choose}); The third and the fourth inequalities are based on (\ref{eqn:delta1choose}) and (\ref{eqn:delta3choose2}), respectively.

Therefore, $\|\mathbf{T}_1(\mathbf{y})-\mathbf{x}_k\|>(1+\epsilon)\|\mathbf{y}-\mathbf{x}_k\|$. If $\mathbf{T}_1(\mathbf{y})\in B(\mathbf{x}_k,\delta_0)$, then $\|\mathbf{T}_1^2(\mathbf{y})-\mathbf{x}_k\|>(1+\epsilon)\|\mathbf{T}_1(\mathbf{y})-\mathbf{x}_k\|>(1+\epsilon)^2 \|\mathbf{y}-\mathbf{x}_k\|$. As long as the current iterate lies in $B(\mathbf{x}_k,\delta_0)$, $\mathbf{T}_1$ will keep on driving it out of $B(\mathbf{x}_k,\delta_0)$. Thus there exists some $s$ such that $\mathbf{T}_1^{s-1}(\mathbf{y})\in B(\mathbf{x}_k,\delta_0)$ and $\mathbf{T}_1^{s}(\mathbf{y})\notin B(\mathbf{x}_k,\delta_0)$ (see Figure \ref{fig:driveout0}).  
\end{proof}

\begin{figure}[!htb]
\centering
\includegraphics[width=0.4\columnwidth]{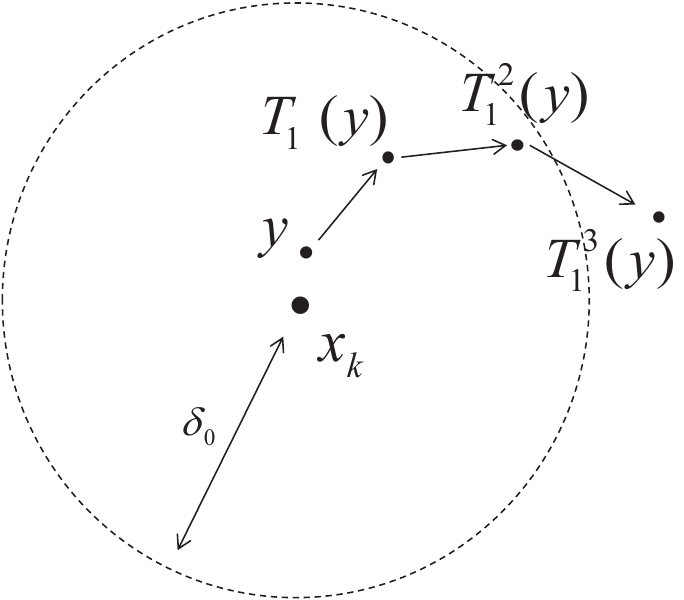}
\caption{Once $\mathbf{y}$ gets into $B(\mathbf{x}_k,\delta_0)$ where $\mathbf{x}_k\ne \mathbf{M}$, $\mathbf{T}_1$ will eventually drive it out of $B(\mathbf{x}_k,\delta_0)$.}
\label{fig:driveout0}
\end{figure}

\subsection{Proof of Theorem \ref{thm:contheorem}}
\label{proof:contheorem}

\begin{proof}
We can assume that $\mathbf{y}_{(p)}$ differs from $\mathbf{M}$ for all $p$. From Lemma \ref{lem:convexhull}, since at most a finite set of iterates do not lie in the convex hull of $\{\mathbf{x}_i\}_{i=1}^m$, the whole sequence $\{\mathbf{y}_{(p)}\}$ is a compact set in $\mathbb{R}^d$. Moreover, by omitting at most a finite number of iterates, we can assume that $\{\mathbf{y}_{(p)}\}\bigcap \{\mathbf{x}_i\}_{i=1}^m=\varnothing$. By the Bolzano-Weierstrass Theorem, there exists a subsequence $\{\mathbf{y}_{(p_v)}\}$ such that $\lim_{v\rightarrow \infty}\mathbf{y}_{(p_v)}=\mathbf{y}_*$ for some $\mathbf{y}_*\in\mathbb{R}^d$. Since the extended operator $\mathbf{T}_1$ is continuous,
\begin{align}
\label{eqn:contingenupdate}
 \lim_{v\rightarrow \infty}\mathbf{T}_1(\mathbf{y}_{(p_v)})=\mathbf{T}_1(\mathbf{y}_*).
\end{align}

According to the decreasing property of $q$PWAWS (\ref{eqn:lqwawsdecrease}), the sequence $C_q(\mathbf{y}_{(p)})$ is bounded below and decreasing, thus it has a limit and any subsequence of $C_q(\mathbf{y}_{(p)})$ should have the same limit. In particular, $C_q(\mathbf{y}_{(p_v)})$ and $C_q(\mathbf{T}_1(\mathbf{y}_{(p_v)}))$ are two subsequences of $C_q(\mathbf{y}_{(p)})$. Hence
\begin{align}
\label{eqn:subsequencelim}
\lim_{v\rightarrow \infty}C_q(\mathbf{T}_1(\mathbf{y}_{(p_v)}))=\lim_{v\rightarrow \infty}C_q(\mathbf{y}_{(p_v)}).\quad
\end{align}
Since $C_q$ is continuous, (\ref{eqn:contingenupdate}) and (\ref{eqn:subsequencelim}) indicate
\begin{align}
\label{eqn:mincharstar}
C_q(\mathbf{T}_1(\mathbf{y}_*))=C_q(\mathbf{y}_*).
\end{align}
If $\mathbf{y}_*\notin \{\mathbf{x}_i\}_{i=1}^m$, then Theorem \ref{thm:nonincreasing} and (\ref{eqn:mincharstar}) indicate $\mathbf{y}_*=\mathbf{T}_1(\mathbf{y}_*)$. By Corollary \ref{cor:charnonsing}, $\mathbf{y}_*=\mathbf{M}$. If $\mathbf{y}_*\in \{\mathbf{x}_i\}_{i=1}^m$, then (\ref{eqn:mincharstar}) trivially holds from (\ref{eqn:t1extend}). To summarize, $\mathbf{y}_*\in\{\mathbf{x}_i\}_{i=1}^m\bigcup \{\mathbf{M}\}$ and only the points in the finite set $\{\mathbf{x}_i\}_{i=1}^m\bigcup \{\mathbf{M}\}$ satisfy (\ref{eqn:mincharstar}) and constitute the fixed points of $\mathbf{T}_1$.

The next step is to prove that if $\mathbf{y}_*=\mathbf{x}_k$ for some $k$, then $\mathbf{x}_k=\mathbf{M}$. If not, we invoke Lemma \ref{lem:kickout} to induce a contradiction. Since $\lim_{v\rightarrow \infty}\mathbf{y}_{(p_v)}=\mathbf{x}_k$, once $\mathbf{y}_{(p_v)}$ gets into $B(\mathbf{x}_k,\delta_0)$, it will be driven out by $\mathbf{T}_1$. Thus for each $\mathbf{y}_{(p_v)}\in B(\mathbf{x}_k,\delta_0)$, there exists a $\mathbf{y}_{(p_u)}\in B(\mathbf{x}_k,\delta_0)$ and a $\mathbf{T}_1(\mathbf{y}_{(p_u)})\notin B(\mathbf{x}_k,\delta_0)$. In other words, $\mathbf{y}_{(p_u)}$ is the iterate that is going to be driven out of $B(\mathbf{x}_k,\delta_0)$. Since $\mathbf{y}_{(p_v)}$ is an infinite sequence converging to $\mathbf{x}_k$, $\mathbf{y}_{(p_u)}$ and $\mathbf{T}_1(\mathbf{y}_{(p_u)})$ are also infinite sequences. By the Bolzano-Weierstrass Theorem, $\mathbf{y}_{(p_u)}$ has a subsequence that converges to some $\mathbf{y}_{*1}$. We still denote this subsequence by $\mathbf{y}_{(p_u)}$. Then $\mathbf{T}_1(\mathbf{y}_{(p_u)})$ also has a subsequence that converges to some $\mathbf{y}_{*2}$ and the subsequence can still be denoted by $\mathbf{T}_1(\mathbf{y}_{(p_u)})$. Note that $\mathbf{y}_{(p_u)}$ and $\mathbf{T}_1(\mathbf{y}_{(p_u)})$ are not necessarily subsequences of $\mathbf{y}_{(p_v)}$. Then
\begin{align}
\label{eqn:ystar12}
&\lim_{u\rightarrow \infty}\mathbf{y}_{(p_u)}=\mathbf{y}_{*1}, \quad\lim_{u\rightarrow \infty}\mathbf{T}_1(\mathbf{y}_{(p_u)})=\mathbf{y}_{*2}.\\
\label{eqn:ystar12b}
&\|\mathbf{y}_{*1}-\mathbf{x}_k\|\leqs \delta_0, \quad\|\mathbf{y}_{*2}-\mathbf{x}_k\|\geqs \delta_0.
\end{align}
Similar to the demonstrations of (\ref{eqn:contingenupdate}),(\ref{eqn:subsequencelim}) and (\ref{eqn:mincharstar}), the accumulation point $\mathbf{y}_{*1}$ is also a fixed point of $\mathbf{T}_1$:
\begin{align}
\label{eqn:ystar12c}
\lim_{u\rightarrow \infty}\mathbf{T}_1(\mathbf{y}_{(p_u)})=\mathbf{T}_1(\lim_{u\rightarrow \infty}\mathbf{y}_{(p_u)})=\mathbf{T}_1(\mathbf{y}_{*1})=\mathbf{y}_{*1}.
\end{align}
From (\ref{eqn:ystar12}), (\ref{eqn:ystar12b}) and (\ref{eqn:ystar12c}),
\begin{align}
\label{eqn:ystar12d}
\mathbf{y}_{*1}=\mathbf{y}_{*2},\quad \|\mathbf{y}_{*1}-\mathbf{x}_k\|= \delta_0.
\end{align}
The process of inducing $\mathbf{y}_{*1}=\mathbf{y}_{*2}$ can be shown as Figure \ref{fig:driveout}.

\begin{figure}[!htb]
\centering
\includegraphics[width=0.4\columnwidth]{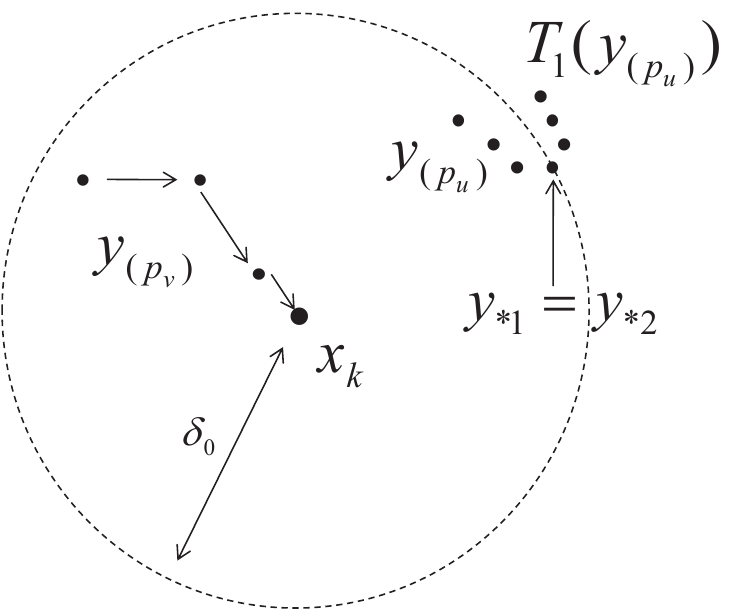}
\caption{The process of inducing $\mathbf{y}_{*1}=\mathbf{y}_{*2}$. $\mathbf{y}_{(p_v)}$ is a subsequence that converges to $\mathbf{x}_k\ne\mathbf{M}$. By Lemma \ref{lem:kickout}, each $\mathbf{y}_{(p_v)}$ induces a $\mathbf{y}_{(p_u)}\in B(\mathbf{x}_k,\delta_0)$ and a $\mathbf{T}_1(\mathbf{y}_{(p_u)})\notin B(\mathbf{x}_k,\delta_0)$. Then the subsequences $\mathbf{y}_{(p_u)}$ and $\mathbf{T}_1(\mathbf{y}_{(p_u)})$ have the same accumulation point $\mathbf{y}_{*1}=\mathbf{y}_{*2}$ at the boundary of $B(\mathbf{x}_k,\delta_0)$.}
\label{fig:driveout}
\end{figure}

For any $0<\delta<\delta_0$, we can apply the same method to obtain a distinct fixed point $\mathbf{y}_{*\delta}$ such that $\mathbf{T}_1(\mathbf{y}_{*\delta})=\mathbf{y}_{*\delta}$ and $\|\mathbf{y}_{*\delta}-\mathbf{x}_k\|= \delta$. Then there are infinite fixed points $\{\mathbf{y}_{*\delta}\}$, which is contradictory to the finite set of fixed points $\{\mathbf{x}_i\}_{i=1}^m\bigcup \{\mathbf{M}\}$. Therefore, $\mathbf{y}_*=\mathbf{x}_k=\mathbf{M}$.

The last step is to prove that the whole sequence $\{\mathbf{y}_{(p)}\}$ converges to $\mathbf{y}_*$. From the above illustrations, the accumulation point $\mathbf{y}_*=\mathbf{M}$. If there is another accumulation point $\tilde{\mathbf{y}}\ne \mathbf{y}_*$, then $\tilde{\mathbf{y}}\in \{\mathbf{x}_i\}_{i=1}^m$. Without loss of generality, suppose $\tilde{\mathbf{y}}=\mathbf{x}_k\ne \mathbf{M}$, then the above method can be repeated to induce an infinite set of fixed points $\{\mathbf{y}_{*\delta}\}$, which leads to a contradiction. Hence, there is only one accumulation point $\mathbf{y}_*=\mathbf{M}$ for the whole sequence $\{\mathbf{y}_{(p)}\}$. Thus $\{\mathbf{y}_{(p)}\}$ and any subsequences converge to $\mathbf{y}_*=\mathbf{M}$.  
\end{proof}

\subsection{Proof of Lemma \ref{lem:orderinfinitesimal}}
\label{proof:orderinfinitesimal}

\begin{proof}
It is straightforward to check from (\ref{eqn:dqgrad}) that $\nabla D_q(\mathbf{y})$ is analytic in some neighborhood $B(\mathbf{x}_k,\delta)$ of $\mathbf{x}_k$ such that $B(\mathbf{x}_k,\delta)\bigcap \{\mathbf{x}_i\}_{i\ne k}=\varnothing$, since the singular component has been excluded from $\nabla D_q(\mathbf{y})$. Furthermore, $\|\nabla D_q(\mathbf{y})\|^2=\nabla D_q(\mathbf{y})^{\top}\nabla D_q(\mathbf{y})$ is also analytic in this neighborhood $B(\mathbf{x}_k,\delta)$. Thus we can adopt the second-order Taylor series expansion of $\|\nabla D_q(\mathbf{y})\|^2$ for $\mathbf{y}_{(p)}\in B(\mathbf{x}_k,\delta)$ at $\mathbf{x}_k$:
\begin{align}
\label{eqn:orderinfinitesimal2}
\|\nabla D_q(\mathbf{y}_{(p)})\|^2 =&\|\nabla D_q(\mathbf{x}_k)\|^2+2\nabla D_q(\mathbf{x}_k)^{\top}H(\mathbf{x}_k)(\mathbf{y}_{(p)}-\mathbf{x}_k){+}\frac{1}{2}(\mathbf{y}_{(p)}{-}\mathbf{x}_k)^{\top}J(\mathbf{x}_k)(\mathbf{y}_{(p)}{-}\mathbf{x}_k){+}o(\|\mathbf{y}_{(p)}{-}\mathbf{x}_k\|^2),
\end{align}
where $H(\mathbf{x}_k)$ and $J(\mathbf{x}_k)$ are the Hessians of $D_q(\mathbf{y})$ and $\|\nabla D_q(\mathbf{y})\|^2$ at $\mathbf{x}_k$, respectively.

Theorem \ref{thm:charsing} indicates that $\nabla D_q(\mathbf{x}_k)=\mathbf{0}$ when $1<q<2$, thus (\ref{eqn:orderinfinitesimal2}) can be further simplified as
\begin{align}
\label{eqn:orderinfinitesimal3}
\|\nabla D_q(\mathbf{y}_{(p)})\|^2 =&\frac{1}{2}(\mathbf{y}_{(p)}{-}\mathbf{x}_k)^{\top}J(\mathbf{x}_k)(\mathbf{y}_{(p)}{-}\mathbf{x}_k){+}o(\|\mathbf{y}_{(p)}{-}\mathbf{x}_k\|^2) \leqs \frac{\vartheta_{J}}{2}\|\mathbf{y}_{(p)}{-}\mathbf{x}_k\|^2{+}o(\|\mathbf{y}_{(p)}{-}\mathbf{x}_k\|^2),
\end{align}
where $\vartheta_{J}$ denotes the largest eigenvalue of $J(\mathbf{x}_k)$. Since $\|\nabla D_q(\mathbf{y}_{(p)})\|^2>0$ when $\mathbf{y}_{(p)}\ne \mathbf{x}_k$, (\ref{eqn:orderinfinitesimal3}) implies that
\begin{align}
\label{eqn:orderinfinitesimal4}
 \frac{\vartheta_{J}}{2}\|\mathbf{y}_{(p)}{-}\mathbf{x}_k\|^2&{+}o(\|\mathbf{y}_{(p)}{-}\mathbf{x}_k\|^2)\geqs\|\nabla D_q(\mathbf{y}_{(p)})\|^2>0 \quad\Longrightarrow \quad \frac{\vartheta_{J}}{2}+o(1)>0 \quad\Longrightarrow \quad\vartheta_{J} \geqs 0.
\end{align}
It means that the inequality in (\ref{eqn:orderinfinitesimal3}) really holds without contradiction.

Next, dividing the leftmost side and the rightmost side of (\ref{eqn:orderinfinitesimal3})
by $\|\mathbf{y}_{(p)}{-}\mathbf{x}_k\|^2$ leads to
\begin{align}
\label{eqn:orderinfinitesimal5}
&\quad\frac{\|\nabla D_q(\mathbf{y}_{(p)})\|^2}{\|\mathbf{y}_{(p)}{-}\mathbf{x}_k\|^2} \leqs \frac{\vartheta_{J}}{2}+o(1) \quad\Longrightarrow \quad \lim_{\mathbf{y}_{(p)}\rightarrow \mathbf{x}_k}\frac{\|\nabla D_q(\mathbf{y}_{(p)})\|^2}{\|\mathbf{y}_{(p)}-\mathbf{x}_k\|^2} \leqs\frac{\vartheta_{J}}{2}.
\end{align}
Since the square root operator $\sqrt{\cdot}$ is continuous in the interval $(0,+\infty)$ and right continuous at $0$, we can take $\sqrt{\cdot}$ inside the limit of (\ref{eqn:orderinfinitesimal5}) and get
\begin{align}
\label{eqn:orderinfinitesimal6}
 &\lim_{\mathbf{y}_{(p)}\rightarrow \mathbf{x}_k}\frac{\|\nabla D_q(\mathbf{y}_{(p)})\|}{\|\mathbf{y}_{(p)}-\mathbf{x}_k\|}{=}\sqrt{\lim_{\mathbf{y}_{(p)}\rightarrow \mathbf{x}_k}\frac{\|\nabla D_q(\mathbf{y}_{(p)})\|^2}{\|\mathbf{y}_{(p)}-\mathbf{x}_k\|^2}} {\leqs}\sqrt{\frac{\vartheta_{J}}{2}}.
\end{align}
Let $\zeta\triangleq\sqrt{\frac{\vartheta_{J}}{2}}$ and the proof is finished.  
\end{proof}

\subsection{Proof of Theorem \ref{thm:rateconver}}
\label{proof:rateconver}

\begin{proof}
Since $\mathbf{y}_{(p)}\rightarrow \mathbf{x}_k$ and the data points are distinct, we can assume that $\mathbf{y}_{(p)}\notin\{\mathbf{x}_i\}_{i=1}^m$, $\forall p \geqs P$ for some sufficiently large $P$. Therefore, $\mathbf{y}_{(p+1)}=\mathbf{T}_1(\mathbf{y}_{(p)})$, $\forall p \geqs P$. We begin with an important equation:
\begin{align}
\label{eqn:rateconverproof}
\mathbf{y}_{(p+1)}{-}\mathbf{x}_k {=}\frac{\sum_{i\ne k} \eta_i^q\| \mathbf{y}_{(p)}-\mathbf{x}_i  \|^{q-2}(\mathbf{x}_i-\mathbf{x}_k)}{\eta_k^q \| \mathbf{y}_{(p)}-\mathbf{x}_k  \|^{q-2}+ \sum_{i\ne k} \eta_i^q\| \mathbf{y}_{(p)}-\mathbf{x}_i  \|^{q-2}}.
\end{align}
The key technique is to eliminate the singular term $\| \mathbf{y}_{(p)}-\mathbf{x}_k  \|^{q-2}$ in the denominator of the right side of (\ref{eqn:rateconverproof}) and construct the rate of convergence simultaneously.

In the $q=1$ case, we divide both sides of (\ref{eqn:rateconverproof}) by the nonzero scalar $\|\mathbf{y}_{(p)}-\mathbf{x}_k\|$:
\begin{align}
\label{eqn:rateconverproof2}
&\frac{\mathbf{y}_{(p+1)}-\mathbf{x}_k}{\|\mathbf{y}_{(p)}-\mathbf{x}_k\|} =\frac{ \sum_{i\ne k} \eta_i\| \mathbf{y}_{(p)}-\mathbf{x}_i  \|^{-1}(\mathbf{x}_i-\mathbf{x}_k)}{\eta_k+\| \mathbf{y}_{(p)}-\mathbf{x}_k  \|\cdot(\sum_{i\ne k} \eta_i\| \mathbf{y}_{(p)}-\mathbf{x}_i  \|^{-1})}.
\end{align}
Taking $L_2$-norm $\|\cdot\|$ on both sides of (\ref{eqn:rateconverproof2}) and letting $\mathbf{y}_{(p)}\rightarrow \mathbf{x}_k$ lead to
\begin{align}
\label{eqn:rateconverproof3}
\lim_{\mathbf{y}_{(p)}\rightarrow \mathbf{x}_k}\frac{\|\mathbf{y}_{(p+1)}-\mathbf{x}_k\|}{\|\mathbf{y}_{(p)}-\mathbf{x}_k\|} =&\lim_{\mathbf{y}_{(p)}\rightarrow \mathbf{x}_k} \frac{ \|\sum_{i\ne k} \eta_i\| \mathbf{y}_{(p)}-\mathbf{x}_i  \|^{-1}(\mathbf{x}_i-\mathbf{x}_k)\|}{\eta_k+\| \mathbf{y}_{(p)}-\mathbf{x}_k  \|\cdot(\sum_{i\ne k} \eta_i\| \mathbf{y}_{(p)}-\mathbf{x}_i  \|^{-1})} \nonumber\\
=&\frac{\|-\nabla D_1(\mathbf{x}_k)\|}{\eta_k+0\cdot(\sum_{i\ne k} \eta_i\| \mathbf{y}_{(p)}-\mathbf{x}_i  \|^{-1})}=\frac{\|\nabla D_1(\mathbf{x}_k)\|}{\eta_k}.
\end{align}
Since $\eta_k>0$, the above limit is well-defined. Based on Theorem \ref{thm:charsing}, the convergence is sublinear, linear or superlinear when $\|\nabla D_1(\mathbf{x}_k)\|=\eta_k$, $0<\|\nabla D_1(\mathbf{x}_k)\|<\eta_k$ or $\|\nabla D_1(\mathbf{x}_k)\|=0$, respectively.

In the $1<q<2$ case, it is a little subtle and we take two steps to eliminate the singular term $\| \mathbf{y}_{(p)}-\mathbf{x}_k  \|^{q-2}$ in the denominator of (\ref{eqn:rateconverproof}). First, we divide both sides of (\ref{eqn:rateconverproof}) by the nonzero scalar $\|\mathbf{y}_{(p)}-\mathbf{x}_k\|$:
\begin{align}
\label{eqn:rateconverproof4}
\frac{\mathbf{y}_{(p+1)}-\mathbf{x}_k}{\|\mathbf{y}_{(p)}-\mathbf{x}_k\|} =&\frac{ \sum_{i\ne k} \eta_i^q\| \mathbf{y}_{(p)}-\mathbf{x}_i  \|^{q-2}(\mathbf{x}_i-\mathbf{x}_k)}{\eta_k^q\| \mathbf{y}_{(p)}{-}\mathbf{x}_k  \|^{q-1}{+}\| \mathbf{y}_{(p)}{-}\mathbf{x}_k  \|\cdot(\sum_{i\ne k} \eta_i^q\| \mathbf{y}_{(p)}{-}\mathbf{x}_i  \|^{q-2})}.
\end{align}
Second, we multiply both the numerator and the denominator of the right side of (\ref{eqn:rateconverproof4}) by $\| \mathbf{y}_{(p)}{-}\mathbf{x}_k  \|^{1-q}$:
\begin{align}
\label{eqn:rateconverproof5}
&\frac{\mathbf{y}_{(p+1)}-\mathbf{x}_k}{\|\mathbf{y}_{(p)}-\mathbf{x}_k\|} \nonumber\\
=&\frac{ \| \mathbf{y}_{(p)}{-}\mathbf{x}_k  \|^{1-q}\cdot(\sum_{i\ne k} \eta_i^q\| \mathbf{y}_{(p)}-\mathbf{x}_i  \|^{q-2}(\mathbf{x}_i-\mathbf{x}_k))}{\eta_k^q{+}\| \mathbf{y}_{(p)}{-}\mathbf{x}_k  \|^{2-q}\cdot(\sum_{i\ne k} \eta_i^q\| \mathbf{y}_{(p)}{-}\mathbf{x}_i  \|^{q-2})} \nonumber\\
=&\frac{ \| \mathbf{y}_{(p)}{-}\mathbf{x}_k  \|^{1-q}\cdot(\sum_{i\ne k} \eta_i^q\| \mathbf{y}_{(p)}-\mathbf{x}_i  \|^{q-2})(\mathbf{y}_{(p)}-\mathbf{x}_k)}{\eta_k^q{+}\| \mathbf{y}_{(p)}{-}\mathbf{x}_k  \|^{2-q}\cdot(\sum_{i\ne k} \eta_i^q\| \mathbf{y}_{(p)}{-}\mathbf{x}_i  \|^{q-2})} -\frac{ \| \mathbf{y}_{(p)}{-}\mathbf{x}_k  \|^{1-q}\cdot\nabla D_q(\mathbf{y}_{(p)})/q}{\eta_k^q{+}\| \mathbf{y}_{(p)}{-}\mathbf{x}_k  \|^{2-q}\cdot(\sum_{i\ne k} \eta_i^q\| \mathbf{y}_{(p)}{-}\mathbf{x}_i  \|^{q-2})}.
\end{align}

Taking $L_2$-norm $\|\cdot\|$ on the leftmost side and the rightmost side of (\ref{eqn:rateconverproof5}) leads to
\begin{align}
\label{eqn:rateconverproof6}
&\frac{\|\mathbf{y}_{(p+1)}-\mathbf{x}_k\|}{\|\mathbf{y}_{(p)}-\mathbf{x}_k\|} \nonumber\\
\leqs&\frac{ \| \mathbf{y}_{(p)}{-}\mathbf{x}_k  \|^{2-q}\cdot(\sum_{i\ne k} \eta_i^q\| \mathbf{y}_{(p)}-\mathbf{x}_i  \|^{q-2})}{\eta_k^q{+}\| \mathbf{y}_{(p)}{-}\mathbf{x}_k  \|^{2-q}\cdot(\sum_{i\ne k} \eta_i^q\| \mathbf{y}_{(p)}{-}\mathbf{x}_i  \|^{q-2})} +\frac{ \| \mathbf{y}_{(p)}{-}\mathbf{x}_k  \|^{1-q}\cdot\|\nabla D_q(\mathbf{y}_{(p)})\|/q}{\eta_k^q{+}\| \mathbf{y}_{(p)}{-}\mathbf{x}_k  \|^{2-q}\cdot(\sum_{i\ne k} \eta_i^q\| \mathbf{y}_{(p)}{-}\mathbf{x}_i  \|^{q-2})}\nonumber\\
=&\frac{ \| \mathbf{y}_{(p)}{-}\mathbf{x}_k  \|^{2-q}\cdot(\sum_{i\ne k} \eta_i^q\| \mathbf{y}_{(p)}-\mathbf{x}_i  \|^{q-2})}{\eta_k^q{+}\| \mathbf{y}_{(p)}{-}\mathbf{x}_k  \|^{2-q}\cdot(\sum_{i\ne k} \eta_i^q\| \mathbf{y}_{(p)}{-}\mathbf{x}_i  \|^{q-2})} +\frac{ \| \mathbf{y}_{(p)}{-}\mathbf{x}_k  \|^{2-q}\cdot\left(\frac{\|\nabla D_q(\mathbf{y}_{(p)})\|}{\| \mathbf{y}_{(p)}{-}\mathbf{x}_k  \|}\right)/q}{\eta_k^q{+}\| \mathbf{y}_{(p)}{-}\mathbf{x}_k  \|^{2-q}\cdot(\sum_{i\ne k} \eta_i^q\| \mathbf{y}_{(p)}{-}\mathbf{x}_i  \|^{q-2})}.
\end{align}

Because $1<q<2$, $0<2-q<1$ and $\| \mathbf{y}_{(p)}{-}\mathbf{x}_k  \|^{2-q}$ is nonsingular when $\mathbf{y}_{(p)}\rightarrow \mathbf{x}_k$. Then we can adopt Lemma \ref{lem:orderinfinitesimal} to dominate the rightmost side of (\ref{eqn:rateconverproof6}):
\begin{align}
\label{eqn:rateconverproof7}
\lim_{\mathbf{y}_{(p)}\rightarrow \mathbf{x}_k}\frac{\|\mathbf{y}_{(p+1)}-\mathbf{x}_k\|}{\|\mathbf{y}_{(p)}-\mathbf{x}_k\|} \leqs&\frac{ 0\cdot(\sum_{i\ne k} \eta_i^q\| \mathbf{x}_k-\mathbf{x}_i  \|^{q-2})+0\cdot\zeta/q}{\eta_k^q{+}0{\cdot}(\sum_{i\ne k} \eta_i^q\| \mathbf{x}_k{-}\mathbf{x}_i  \|^{q-2})} =0,
\end{align}
which shows a superlinear convergence for $1<q<2$.  
\end{proof}

\end{document}